\documentclass
[
    a4paper,
    DIV=11,
    abstract=true,
]
{scrartcl}

\usepackage
{
    amsmath,
    amssymb,
    amsthm,
    authblk,
    dsfont, 
    enumitem,
    graphicx,
    mathtools,
    nicefrac,
    tabularx,
    todonotes,
    tikz,
    xcolor,
}

\usepackage[scr=boondoxupr]{mathalfa}

\usepackage[utf8]{inputenc}

\usepackage[pdffitwindow=false,
            plainpages=false,
            pdfpagelabels=true,
            pdfpagemode=UseOutlines,
            pdfpagelayout=SinglePage,
            bookmarks=false,
            colorlinks=true,
            hyperfootnotes=false,
            linkcolor=blue,
            urlcolor=blue!30!black,
            citecolor=green!50!black]{hyperref}

\usepackage[bf,normal]{caption}
\usepackage{color}

\usetikzlibrary{arrows, arrows.meta, shapes, fit, automata, positioning}

\graphicspath{{pics/}}

\DeclareMathAlphabet{\mathpzc}{OT1}{pzc}{m}{it}

\newcommand{\subfiguretitle}[1]{{\scriptsize{#1}} \\}
\newcommand{\R}{\mathbb{R}}                                      
\newcommand{\innerprod}[2]{\left\langle #1,\, #2 \right\rangle}  
\newcommand{\ts}{\hspace*{0.1em}}                                
\newcommand{\mc}[2][]{\mathpzc{#2}{\smash[t]{\mathstrut}}_{#1}}  
\providecommand{\norm}[1]{\left\lVert #1 \right\rVert}           

\newcolumntype{C}[1]{>{\centering\let\newline\\\arraybackslash\hspace{0pt}}m{#1}}

\newcommand\xqed[1]{\leavevmode\unskip\penalty9999 \hbox{}\nobreak\hfill \quad\hbox{#1}}
\newcommand{\exampleSymbol}{\xqed{$\triangle$}}

\DeclareMathOperator{\diag}{diag}

\newtheorem{theorem}{Theorem}[section]
\newtheorem{corollary}[theorem]{Corollary}
\newtheorem{lemma}[theorem]{Lemma}
\newtheorem{proposition}[theorem]{Proposition}
\newtheorem{definition}[theorem]{Definition}
\theoremstyle{definition}
\newtheorem{example}[theorem]{Example}
\newtheorem{remark}[theorem]{Remark}
\newtheorem{textalgorithm}[theorem]{Algorithm}

\makeatletter
\renewcommand*\env@matrix[1][*\c@MaxMatrixCols c]{%
  \hskip -\arraycolsep
  \let\@ifnextchar\new@ifnextchar
  \array{#1}}
\makeatother

\mathtoolsset{centercolon} 

\allowdisplaybreaks

\DeclareCaptionLabelFormat{period}{#1~#2}
\makeatletter
\def\blfootnote{\gdef\@thefnmark{}\@footnotetext}
\makeatother

\begin{document}

\title{Transfer operators on graphs: \\ Spectral clustering and beyond}
\author[1]{Stefan Klus}
\author[2]{Maia Trower}
\affil[1]{School of Mathematical \& Computer Sciences, Heriot--Watt University, UK}
\affil[2]{Maxwell Institute for Mathematical Sciences, University of Edinburgh and Heriot--Watt University, UK}

\date{}

\maketitle

\begin{abstract}
Graphs and networks play an important role in modeling and analyzing complex interconnected systems such as transportation networks, integrated circuits, power grids, citation graphs, and biological and artificial neural networks. Graph clustering algorithms can be used to detect groups of strongly connected vertices and to derive coarse-grained models. We define transfer operators such as the Koopman operator and the Perron--Frobenius operator on graphs, study their spectral properties, introduce Galerkin projections of these operators, and illustrate how reduced representations can be estimated from data. In particular, we show that spectral clustering of undirected graphs can be interpreted in terms of eigenfunctions of the Koopman operator and propose novel clustering algorithms for directed graphs based on generalized transfer operators. We demonstrate the efficacy of the resulting algorithms on several benchmark problems and provide different interpretations of clusters.
\end{abstract}

\section{Introduction}

Transfer operators are frequently used for the analysis of complex dynamical systems since their eigenvalues and eigenfunctions contain important information about global properties of the underlying processes~\cite{Ko31, LaMa94, DJ99, Mezic05}. While the eigenvalues are related to inherent relaxation timescales, the corresponding eigenfunctions can be used to identify metastable sets representing, for example, stable conformations of proteins or coherent sets describing, for instance, slowly mixing eddies in the ocean~\cite{FrSaMo10, SS13, KKS16, BaKo17, KHMN19, WuNo20, SKH23}. Other applications include model reduction, system identification, change-point detection, forecasting, control, and stability analysis~\cite{MauGon16, MauMez16, KorMez16, Gia19, KNPNCS20}. Although transfer operators have been known for a long time, they recently gained renewed attention due to the availability of large data sets and powerful machine learning methods to estimate these operators and their spectral properties from simulation or measurement data.

Spectral clustering for graphs is a well-known and powerful technique for partitioning networks into groups of nodes that are well-connected internally, and poorly connected to other groups of nodes~\cite{Luxburg07}. It was first shown in~\cite{DonHof73} that graphs can be partitioned based on eigenvectors of the adjacency matrix, but more recently spectral clustering has become increasingly popular in the machine learning community. In particular, the spectral clustering algorithms proposed in~\cite{ShiMal00, NgJorWei02} are commonly deployed and are capable of identifying near-optimal clusters in undirected static graphs~\cite{PenSunZan15}. Many extensions of these standard algorithms exist, including extensions to directed graphs~\cite{SaPa11} and time-evolving graphs \cite{KD22, NXCH2010}. Recent work includes scalable algorithms that can efficiently cluster large graphs~\cite{LiuWanDan13}, and a regularized method that is able to interpret heterogeneous structures in graphs~\cite{SenChe15, LiKaoRen19}.

The underlying technique can be derived in a number of ways, for example by viewing the problem as a graph partitioning or from a random walk perspective. In ~\cite{Luxburg07}, it is shown that using a graph partition viewpoint allows the clustering problem to be reformulated as a relaxed mincut problem with constraints to balance the clusters. In particular, balancing the clusters by the number of vertices per cluster gives an objective function known as the RatioCut, and relaxing this problem is equivalent to standard spectral clustering using the unnormalized graph Laplacian. Relaxing the normalized cut objective function, which balances clusters by weight, leads instead to the spectral clustering formulation using the normalized graph Laplacian.

As mentioned above, random walks can also be used to reformulate the graph clustering problem. Intuitively, this corresponds to finding a graph partition such that a random walker will remain within a cluster for long periods of time and will rarely jump between clusters. This random-walk interpretation of spectral clustering methods for undirected graphs is well-known, see, e.g.,~\cite{MS01, Luxburg07}. Nevertheless, the dynamical systems perspective and relationships with transfer operators and their eigenvalues and eigenfunctions have to our knowledge neither been used to analyze existing methods in detail, nor to systematically generalize spectral clustering to directed or dynamic graphs. In~\cite{KD22}, we defined transfer operators on graphs and showed that spectral clustering for undirected graphs is equivalent to detecting metastable sets in reversible stochastic dynamical systems. The eigenvalues of transfer operators associated with non-reversible systems, however, are in general complex-valued and the notion of metastability is not suitable for such problems anymore. This motivates the definition of so-called \emph{coherent sets}, which can be regarded as time-dependent metastable sets~\cite{FrSaMo10, BaKo17, KWNS18}. Coherent sets can be detected by computing eigenvalues and eigenfunctions of generalized transfer operators based on the forward--backward dynamics of the system. By defining these operators on graphs, it is possible to develop spectral clustering techniques for directed and time-evolving graphs, thus establishing and exploiting links between dynamical systems theory and graph theory. Applying methods originally developed for the analysis of complex dynamical systems to graphs allows for a rigorous derivation and, at the same time, also intuitive interpretation of spectral clustering techniques.

In terms of existing methods for clustering directed graphs, several proposed approaches are based on symmetrization techniques~\cite{SaPa11}. A simple example is to define a new graph with adjacency matrix $ A + A^\top $, where $ A $ is the adjacency matrix of the original directed graph. This is essentially equivalent to ignoring the directions of the edges. It is also possible to symmetrize the graph by defining an adjacency matrix in terms of the stationary distribution of a random walk~\cite{Gle06}. Both of these proposals have a notable drawback; they fail to recognize similar nodes with common in-links or out-links, and can only cluster based on the existing edges in the graph with directional information removed. Our clustering approach is loosely related to these symmetrization techniques, in that it also leads to real-valued spectra since the resulting operators are self-adjoint (with respect to potentially weighted inner products). However, we do not directly symmetrize the adjacency matrix or other graph-related matrices, but exploit properties of random walk processes and associated transfer operators. A more detailed comparison and numerical results will be presented below.

In this work, we will generalize the definition of transfer operators on graphs, first introduced in~\cite{KD22}, and show how eigenfunctions of these operators are related to classical spectral clustering approaches. The main contributions are:
\begin{itemize}[leftmargin=2ex, itemsep=0ex, topsep=0.5ex]
\item We study the spectral properties of graph transfer operators and their matrix representations and show that the eigenfunctions of these operators can also be obtained by solving associated optimization problems.
\item We illustrate how transfer operators can be used to detect clusters or community structures in directed graphs and provide illustrative interpretations of the identified clusters in terms of coherent sets and block matrices.
\item We define Galerkin projections of transfer operators, show how these coarse-grained operators can be estimated from random-walk data, and analyze the convergence of data-driven approximations.
\item Furthermore, we highlight additional applications of the proposed methods such as graph drawing or network inference, apply our clustering approach to different types of benchmark problems, and compare it with other spectral clustering algorithms.
\end{itemize}
In Section~\ref{sec:Transer operators}, we will define transfer operators, analyze their properties, and highlight relationships with spectral clustering algorithms. Section~\ref{sec:Analysis and approximation} illustrates how transfer operators can be approximated and estimated from data. Additionally, we will point out different interpretations of clustering methods. Numerical results will be presented in Section~\ref{sec:Numerical results} and concluding remarks and open questions in Section~\ref{sec:Conclusion}.

\section{Graphs, transfer operators, and spectral clustering}
\label{sec:Transer operators}

In this section, we will introduce the mathematical tools required for the transfer operator- based identification of clusters in graphs.

\subsection{Directed and undirected graphs}

In what follows, we will mainly consider directed graphs, but the derived methods can, of course, also be applied to undirected graphs. Relationships with well-known clustering techniques for undirected graphs will be discussed below.

\begin{definition}[Directed graph]\label{def: directed_graph}
A \emph{directed graph} $ \mc{G} = (\mc{V}, \mc{E}) $ is given by a set of vertices $ \mc{V} = \{ \mc[1]{v}, \dots, \mc[n]{v} \} $ and edges $ \mc{E} \subseteq \mc{V} \times \mc{V} $.
\end{definition}

The \emph{weighted adjacency matrix} $ A \in \R^{n \times n} $ associated with a graph $ \mc{G} $ is defined by
\begin{equation*}
    a_{ij} =
    \begin{cases}
        w(\mc[i]{v}, \mc[j]{v}), & \text{if } (\mc[i]{v}, \mc[j]{v}) \in \mc{E}, \\
        0, & \text{otherwise},
    \end{cases}
\end{equation*}
where $ w(\mc[i]{v}, \mc[j]{v}) > 0 $ is the weight of the corresponding edge. If the graph is undirected, then the adjacency matrix is symmetric. Additionally, we introduce the row-stochastic \emph{transition probability matrix} $ S = D_\mathscr{o}^{-1} A \in \R^{n \times n} $, with
\begin{equation*}
    D_\mathscr{o} = \diag\big(\mathscr{o}(\mc[1]{v}), \dots, \mathscr{o}(\mc[n]{v})\big)
    \quad \text{and} \quad
    \mathscr{o}(\mc[i]{v}) = \sum_{j=1}^n a_{ij}.
\end{equation*}
That is, $ s_{ij} $ is the probability that a random walker starting in $ \mc[i]{v} $ will go to $ \mc[j]{v} $ in one step. We refer to $ D_\mathscr{o} $ as the matrix of out-degrees, and we will later also make use of the in-degree matrix $ D_\mathscr{i} $, which is defined in a similar way, with
\begin{equation*}
    D_\mathscr{i} = \diag\big(\mathscr{i}(\mc[1]{v}), \dots, \mathscr{i}(\mc[n]{v})\big)
    \quad \text{and} \quad
    \mathscr{i}(\mc[i]{v}) = \sum_{j=1}^n a_{ji}.
\end{equation*}

\subsection{Benchmark problems}

In order to generate benchmark problems with clearly defined clusters, we will use the so-called \emph{directed stochastic block model}.

\begin{definition}[Directed stochastic block model]
The graph $ \mc{G} $ is said to be sampled from the \emph{directed stochastic block model} (DSBM) $ \mathbf{G}(r_b, n_b, E) $, with $ E \in [0, 1]^{r_b \times r_b} $, if the adjacency matrix
\begin{equation*}
    A =
    \begin{bmatrix}
        A_{11} & \dots & A_{1 r_b} \\
        \vdots & \ddots & \vdots \\
        A_{r_b 1} & \dots & A_{r_b r_b}
    \end{bmatrix}
    \in \R^{(r_b \ts n_b) \times (r_b \ts n_b)}
\end{equation*}
is a block matrix whose blocks $ A_{ij} \in R^{n_b \times n_b} $ contain positive entries with probability $ F_{ij} $. Here, $ r_b $ is the number and $ n_b $ the size of the blocks.
\end{definition}

Our definition differs slightly from the DSBM described in \cite{CLSZ20}, in that we allow different probabilities for all the blocks. Undirected graphs can be constructed in a similar way. While the DSBM is a frequently used model to generate benchmark problems, real-world graphs typically exhibit a more complicated structure. An algorithm to generate weighted directed graphs with heterogeneous node-degree distributions and cluster sizes for testing community detection methods is described in \cite{LF09}. Briefly, the graphs in~\cite{LF09} are constructed by randomly assigning in-degrees to each node from a power law distribution and assigning out-degrees from a delta distribution. A mixing parameter is introduced for each of these degrees that is related to the quality of the graph partition---that is, the lower the mixing parameter the better the partitioning of the graph. Cluster sizes are also drawn from a power law distribution and a subgraph is constructed from each of these clusters. The subgraphs are finally joined randomly, with a rewiring process to ensure that the in-degree and out-degree distributions are preserved. We will use this algorithm to generate a set of graphs with different characteristics.

\subsection{Coherent set illustration}

The spectral clustering approach for directed and time-evolving graphs proposed in \cite{KD22} is based on a generalized Laplacian that contains information about \emph{coherent sets}. The detection of such coherent sets requires the notion of transfer operators, which describe the evolution of probability densities or observables. Before we introduce these operators, let us first illustrate the definition of coherence---defined in \cite{Froyland13, BaKo17} for continuous dynamical systems---for directed graphs.

\begin{definition}[Coherent pair]
Let $ S^\tau $ be the flow associated with a dynamical system and $ \tau $ a fixed lag time. Two sets $ \mathbb{A} $ and $ \mathbb{B} $ form a so-called \emph{coherent pair} if $ S^\tau(\mathbb{A}) \approx \mathbb{B} $ and $ S^{-\tau}(\mathbb{B}) \approx \mathbb{A} $.
\end{definition}

That is, the set $ \mathbb{A} $ is almost invariant under the forward--backward dynamics and we call it a \emph{finite-time coherent set}. To avoid that $ (S^{-\tau} \circ S^\tau)(\mathbb{A}) = \mathbb{A} $ for all sets $ \mathbb{A} $, a small random perturbation of the dynamics is required for deterministic systems, see \cite{BaKo17}. In our setting, the dynamics are given by a random walk process on the graph and thus already non-deterministic.

\begin{example} \label{ex:Coherent set illustration}
Figure~\ref{fig:Coherent set illustration} illustrates the difference between a coherent set (green) and a set that is dispersed by the random-walk dynamics (red). The example shows that coherent sets form weakly coupled clusters and can be regarded as a natural generalization of clusters in undirected graphs. \exampleSymbol

\begin{figure}
    \centering
    \begin{minipage}[t]{0.3\textwidth}
        \centering
        \subfiguretitle{(a)}
        \vspace*{1ex}
        \resizebox{0.85\textwidth}{!}{%
        \begin{tikzpicture}[
                >= stealth, 
                semithick 
            ]
            \tikzstyle{every state}=[
                draw=black,
                thick,
                fill=white,
                inner sep=0pt,
                text width=6mm,
                align=center,
                scale=0.6
            ]

            \node[state] (v3) {3};
            \node[state] (v4) [right=0.8cm of v3] {4};
            \node[state] (v1) [below=0.8cm of v4] {1};
            \node[state] (v2) [below=0.8cm of v3] {2};

            \node[state] (v8) [right=1.35cm of v4] {8};
            \node[state] (v5) [above left=0.55cm and 0.55cm of v8] {5};
            \node[state] (v6) [above=1.35cm of v8] {6};
            \node[state] (v7) [right=1.3cm of v5] {7};

            \node[state] (v9) [right=1.35cm of v1] {9};
            \node[state] (v10) [below right=0.55cm and 0.55cm of v9] {10};
            \node[state] (v11) [below=1.35cm of v9] {11};
            \node[state] (v12) [left=1.35cm of v10] {12};

            \path[->] (v1) edge node {} (v2);
            \path[->] (v2) edge node {} (v3);
            \path[->] (v3) edge node {} (v4);
            \path[->] (v4) edge node {} (v1);

            \path[->] (v5.45) edge node {} (v6.225);
            \path[->] (v6.-45) edge node {} (v7.135);
            \path[->] (v7.-135) edge node {} (v8.45);
            \path[->] (v8.135) edge node {} (v5.-45);

            \path[->] (v9.-45) edge node {} (v10.135);
            \path[->] (v10.-135) edge node {} (v11.45);
            \path[->] (v11.135) edge node {} (v12.-45);
            \path[->] (v12.45) edge node {} (v9.225);

            \path[->, dashed] (v4) edge node [dotted] {} (v5);
            \path[->, dashed] (v8) edge node [dotted] {} (v9);
            \path[->, dashed] (v12) edge node [dotted] {} (v1);

            \foreach \Point in {(-0.110, -1.589), (1.474, 0.139), (0.247, -1.350), (1.318, 0.020), (-0.101, -0.061), (1.334, -0.083), (1.438, -1.512), (1.467, 0.080), (1.504, -0.002), (1.462, -1.220), (0.205, -1.418), (0.177, -1.551), (-0.046, 0.243), (1.353, -1.456), (-0.123, 0.079), (-0.211, -0.080), (0.057, -0.000), (0.062, -1.587), (1.468, -1.529), (0.022, -0.157), (0.078, -1.431), (1.466, -1.454), (0.033, -1.295), (-0.188, -1.547), (-0.183, 0.085), (1.441, -0.070), (0.150, 0.014), (1.241, -0.102), (1.261, -1.305), (1.340, -1.517), (-0.058, -1.484), (1.485, 0.048), (1.430, -0.142), (0.004, -1.495), (0.054, -1.482), (1.293, -1.507), (0.087, 0.098), (-0.176, -0.015), (1.425, -1.461), (0.076, -1.375), (1.385, -1.456), (1.468, 0.085), (0.047, -1.424), (1.492, -0.091), (1.414, -1.260), (1.381, -0.084), (0.194, -0.025), (0.044, -0.080), (0.029, 0.047), (0.000, 0.033), (1.520, -1.540), (1.421, -0.199), (1.403, -1.207), (0.194, -1.401), (1.366, -1.347), (1.506, 0.001), (1.561, -1.482), (-0.200, -0.075), (-0.020, -0.037), (-0.085, 0.134), (1.238, -0.049), (1.303, -1.447), (1.333, 0.048), (-0.104, -0.082), (0.010, 0.108), (1.215, -1.220), (-0.088, 0.065), (0.105, 0.132), (1.348, 0.002), (-0.216, -1.461), (1.427, -1.364), (1.365, -1.387), (0.267, -0.001), (1.471, -0.046), (-0.153, -0.035), (0.056, -1.480), (1.215, 0.005), (-0.044, 0.128), (1.458, 0.173), (1.490, -1.533), (1.248, -1.350), (1.431, -0.094), (1.524, 0.003), (1.412, 0.027), (0.075, -1.248), (1.374, -1.573), (1.501, -0.093), (1.524, -0.059), (0.149, -1.481), (1.289, -1.304), (-0.136, -1.332), (1.444, -0.014), (1.366, 0.206), (-0.037, -0.219), (-0.094, -0.081), (-0.058, -1.229), (-0.039, -0.069), (1.355, 0.071), (-0.132, -1.337), (-0.133, -1.413)}
                \draw[green,fill=green] \Point circle (0.1ex);

            \foreach \Point in {(3.444, 0.041), (3.266, 0.051), (3.297, -0.016), (4.349, -2.513), (4.211, -2.395), (4.327, -2.373), (3.134, -0.048), (4.433, 1.007), (4.203, -2.208), (3.136, -1.309), (3.134, -1.340), (4.082, 0.953), (3.219, -1.377), (4.183, 0.931), (3.523, -0.101), (4.335, 0.898), (3.310, 0.127), (4.448, -2.306), (4.415, -2.209), (3.333, 0.040), (3.173, 0.092), (3.174, -1.491), (4.174, -2.505), (4.309, 0.931), (4.144, 0.955), (4.076, 0.853), (3.381, -1.573), (4.450, -2.481), (3.344, -0.072), (3.333, -1.359), (4.372, 1.097), (3.307, -0.018), (4.324, -2.184), (3.440, -1.564), (3.277, -1.313), (4.206, -2.278), (3.303, 0.160), (3.403, -1.368), (3.560, -0.039), (4.337, -2.289), (4.236, -2.331), (4.158, -2.247), (3.540, -1.369), (4.181, -2.162), (3.222, -0.127), (4.164, 0.948), (4.111, 1.100), (4.507, -2.262), (4.318, 0.890), (3.258, -1.232), (4.372, -2.460), (4.272, -2.361), (3.414, 0.098), (4.350, 1.073), (3.380, -1.232), (4.166, 1.183), (3.418, 0.146), (4.278, -2.427), (4.304, 1.023), (4.349, -2.329), (3.418, 0.050), (3.273, -0.137), (4.296, 0.992), (4.300, 0.827), (4.535, -2.330), (4.384, 0.855), (4.513, -2.328), (3.379, -0.081), (3.168, 0.025), (3.507, 0.049), (4.457, 0.815), (3.449, -1.296), (3.340, -0.095), (3.315, -1.434), (4.098, -2.326), (3.422, -1.408), (3.409, -0.010), (4.479, -2.420), (4.323, 1.201), (4.192, 0.982), (3.427, 0.009), (4.345, 1.005), (4.467, -2.372), (4.077, 1.054), (4.214, 0.980), (3.413, -0.194), (4.214, 1.034), (4.072, -2.462), (3.268, -0.049), (3.417, -0.133), (3.537, 0.110), (3.397, -1.394), (3.406, -1.242), (4.313, -2.127), (3.530, -1.448), (4.261, -2.189), (4.175, 0.809), (3.377, 0.044), (4.270, 0.879), (3.311, -1.297)}
                \draw[red,fill=red] \Point circle (0.1ex);
        \end{tikzpicture}}
    \end{minipage}
    \begin{minipage}[t]{0.3\textwidth}
        \centering
        \subfiguretitle{(b)}
        \vspace*{1ex}
        \resizebox{0.85\textwidth}{!}{%
        \begin{tikzpicture}[
                >= stealth, 
                semithick 
            ]
            \tikzstyle{every state}=[
                draw = black,
                thick,
                fill = white,
                inner sep=0pt,
                text width=6mm,
                align=center,
                scale=0.6
            ]

            \node[state] (v3) {3};
            \node[state] (v4) [right=0.8cm of v3] {4};
            \node[state] (v1) [below=0.8cm of v4] {1};
            \node[state] (v2) [below=0.8cm of v3] {2};

            \node[state] (v8) [right=1.35cm of v4] {8};
            \node[state] (v5) [above left=0.55cm and 0.55cm of v8] {5};
            \node[state] (v6) [above=1.35cm of v8] {6};
            \node[state] (v7) [right=1.3cm of v5] {7};

            \node[state] (v9) [right=1.35cm of v1] {9};
            \node[state] (v10) [below right=0.55cm and 0.55cm of v9] {10};
            \node[state] (v11) [below=1.35cm of v9] {11};
            \node[state] (v12) [left=1.35cm of v10] {12};

            \path[->] (v1) edge node {} (v2);
            \path[->] (v2) edge node {} (v3);
            \path[->] (v3) edge node {} (v4);
            \path[->] (v4) edge node {} (v1);

            \path[->] (v5.45) edge node {} (v6.225);
            \path[->] (v6.-45) edge node {} (v7.135);
            \path[->] (v7.-135) edge node {} (v8.45);
            \path[->] (v8.135) edge node {} (v5.-45);

            \path[->] (v9.-45) edge node {} (v10.135);
            \path[->] (v10.-135) edge node {} (v11.45);
            \path[->] (v11.135) edge node {} (v12.-45);
            \path[->] (v12.45) edge node {} (v9.225);

            \path[->, dashed] (v4) edge node [dotted] {} (v5);
            \path[->, dashed] (v8) edge node [dotted] {} (v9);
            \path[->, dashed] (v12) edge node [dotted] {} (v1);

            \foreach \Point in {(0.065, -0.139), (-0.134, 0.113), (0.000, 0.038), (1.310, -0.020), (0.031, 0.017), (1.579, -1.462), (-0.015, -1.317), (1.495, -0.197), (1.361, -1.468), (-0.002, -1.348), (1.444, 0.114), (1.445, 0.032), (1.495, -1.357), (2.218, 0.986), (1.235, -1.289), (-0.211, -1.455), (0.009, -1.404), (0.188, -1.394), (1.339, -1.341), (-0.121, -1.345), (0.025, -0.039), (1.338, -0.119), (1.212, -1.442), (0.078, -1.471), (0.158, -1.489), (0.214, -1.469), (1.410, -1.257), (1.523, -1.476), (0.074, -1.511), (1.229, -1.334), (1.337, -0.214), (-0.006, 0.139), (1.522, -0.150), (1.404, 0.080), (0.097, 0.156), (1.559, -1.314), (1.258, 0.004), (-0.138, 0.123), (1.374, 0.200), (0.084, -1.434), (-0.048, -0.089), (1.526, -1.362), (1.388, 0.042), (0.101, -0.087), (1.321, -1.331), (0.159, 0.125), (-0.012, -1.533), (1.266, -1.493), (1.513, -1.455), (1.428, -1.379), (1.250, -1.275), (0.202, -0.004), (1.600, -0.076), (0.060, 0.212), (0.131, -0.070), (-0.102, 0.005), (1.376, -0.006), (0.127, 0.022), (0.166, -0.004), (1.387, 0.079), (0.043, -0.021), (-0.159, -1.512), (1.462, 0.054), (-0.147, -1.388), (0.053, -0.118), (-0.100, 0.119), (-0.175, 0.045), (1.498, -1.579), (1.387, -1.518), (0.009, -1.442), (1.25, 0.122), (0.060, -0.062), (0.053, -1.534), (-0.123, -1.515), (1.389, -0.052), (0.191, -0.007), (0.056, 0.057), (1.437, -1.381), (0.124, -1.516), (0.164, 0.069), (-0.115, -1.289), (0.118, -1.229), (-0.042, -1.236), (1.426, -0.146), (-0.078, -1.307), (0.179, -0.118), (-0.011, -1.466), (0.188, -0.090), (1.533, -0.062), (0.078, -0.132), (1.250, -0.120), (1.432, -1.358), (1.318, -0.097), (1.461, -1.504), (-0.003, -1.622), (1.232, -1.404), (1.512, 0.165), (1.468, -1.351), (1.341, -1.402), (1.559, -1.540)}
                \draw[green,fill=green] \Point circle (0.1ex);

            \foreach \Point in {(4.190, -2.357), (4.271, 0.957), (3.166, -3.515), (4.291, 0.938), (2.238, 1.115), (4.410, -2.287), (3.348, -3.364), (4.480, 0.882), (3.433, -1.348), (2.411, 0.966), (3.530, 0.109), (4.147, -2.318), (4.358, -2.367), (3.369, 0.130), (3.372, 0.170), (4.160, 0.929), (4.424, 1.022), (4.341, -2.256), (3.410, 0.032), (4.370, -2.216), (3.431, -1.563), (2.262, 0.782), (2.539, 0.941), (3.342, -3.528), (3.273, -1.305), (3.242, -3.175), (3.201, -0.124), (3.393, -3.502), (3.103, -1.350), (4.130, 0.925), (2.301, 1.099), (4.065, 0.981), (4.318, -2.555), (4.189, 0.981), (4.212, 1.105), (3.361, -0.253), (3.140, -0.023), (3.172, -0.062), (2.300, 0.941), (4.278, 0.982), (4.380, -2.366), (4.409, -2.268), (3.372, 0.194), (3.340, 0.049), (3.362, -3.218), (4.170, -2.440), (3.319, -0.013), (4.431, 0.926), (2.265, 0.822), (3.188, -3.272), (3.141, 0.031), (3.357, -1.354), (3.243, -1.376), (3.351, -3.268), (4.190, 1.161), (2.457, 0.994), (2.353, 0.880), (2.411, 1.030), (4.385, -2.352), (3.441, -0.061), (3.194, -3.452), (3.498, 0.065), (4.275, -2.375), (4.534, -2.377), (3.507, -3.355), (3.318, -0.056), (3.306, -0.012), (3.146, -1.313), (3.333, -1.236), (3.463, -0.116), (3.406, -3.457), (4.196, -2.386), (3.203, -0.114), (4.315, -2.480), (4.141, -2.208), (3.454, 0.113), (4.303, 0.768), (2.482, 1.142), (4.248, -2.147), (3.183, -1.559), (3.292, -0.052), (4.351, -2.278), (4.279, 0.870), (3.420, -0.145), (4.499, -2.242), (3.248, -0.055), (3.450, 0.220), (3.315, -0.032), (3.403, -1.426), (4.214, -2.475), (2.184, 1.062), (4.491, 1.007), (3.384, -0.080), (3.403, -3.231), (3.187, 0.036), (3.367, 0.030), (2.349, 1.015), (2.240, 1.025), (3.369, -1.436), (4.449, -2.503)}
                \draw[red,fill=red] \Point circle (0.1ex);
        \end{tikzpicture}}
    \end{minipage}
    \begin{minipage}[t]{0.3\textwidth}
        \centering
        \subfiguretitle{(c)}
        \vspace*{1ex}
        \resizebox{0.85\textwidth}{!}{%
        \begin{tikzpicture}[
                >= stealth, 
                semithick 
            ]
            \tikzstyle{every state}=[
                draw = black,
                thick,
                fill = white,
                inner sep=0pt,
                text width=6mm,
                align=center,
                scale=0.6
            ]

            \node[state] (v3) {3};
            \node[state] (v4) [right=0.8cm of v3] {4};
            \node[state] (v1) [below=0.8cm of v4] {1};
            \node[state] (v2) [below=0.8cm of v3] {2};

            \node[state] (v8) [right=1.35cm of v4] {8};
            \node[state] (v5) [above left=0.55cm and 0.55cm of v8] {5};
            \node[state] (v6) [above=1.35cm of v8] {6};
            \node[state] (v7) [right=1.3cm of v5] {7};

            \node[state] (v9) [right=1.35cm of v1] {9};
            \node[state] (v10) [below right=0.55cm and 0.55cm of v9] {10};
            \node[state] (v11) [below=1.35cm of v9] {11};
            \node[state] (v12) [left=1.35cm of v10] {12};

            \path[->] (v1) edge node {} (v2);
            \path[->] (v2) edge node {} (v3);
            \path[->] (v3) edge node {} (v4);
            \path[->] (v4) edge node {} (v1);

            \path[->] (v5.45) edge node {} (v6.225);
            \path[->] (v6.-45) edge node {} (v7.135);
            \path[->] (v7.-135) edge node {} (v8.45);
            \path[->] (v8.135) edge node {} (v5.-45);

            \path[->] (v9.-45) edge node {} (v10.135);
            \path[->] (v10.-135) edge node {} (v11.45);
            \path[->] (v11.135) edge node {} (v12.-45);
            \path[->] (v12.45) edge node {} (v9.225);

            \path[->, dashed] (v4) edge node [dotted] {} (v5);
            \path[->, dashed] (v8) edge node [dotted] {} (v9);
            \path[->, dashed] (v12) edge node [dotted] {} (v1);

            \foreach \Point in {(1.325, -1.357), (1.357, 0.012), (1.442, -1.360), (3.498, -0.029), (1.153, 0.025), (-0.093, 0.053), (1.267, 0.102), (-0.090, -0.066), (1.440, 0.235), (-0.095, 0.179), (1.449, -1.489), (1.476, 0.022), (-0.047, -0.144), (2.345, -2.366), (0.053, -1.374), (1.367, -1.293), (-0.040, 0.062), (-0.116, -1.419), (1.429, -1.328), (1.404, 0.034), (0.126, 0.124), (1.344, 0.023), (0.155, -1.383), (1.403, 0.142), (1.489, -0.066), (0.000, -1.512), (1.375, 0.090), (-0.129, -1.298), (1.362, -1.470), (-0.073, -1.613), (1.391, -1.515), (1.232, -1.400), (-0.059, -0.040), (0.212, 0.016), (-0.165, -0.029), (0.160, -0.176), (1.306, -0.131), (1.287, -1.347), (0.121, 0.190), (0.118, -0.076), (-0.106, -0.186), (1.379, 0.069), (0.150, -1.392), (1.354, -1.378), (-0.082, -1.391), (1.392, -0.124), (-0.034, -1.354), (0.142, -1.345), (1.594, -1.469), (0.019, -1.417), (0.015, -0.108), (1.390, -0.032), (-0.037, -0.036), (-0.026, 0.001), (1.530, -0.038), (-0.001, -1.565), (0.019, -1.392), (-0.065, -1.363), (-0.029, -0.227), (-0.113, -1.312), (-0.092, -1.389), (0.112, 0.072), (1.474, -0.136), (1.529, 0.057), (1.305, -0.099), (-0.241, -0.054), (0.097, -1.298), (1.421, -1.297), (1.411, -1.267), (0.122, -0.127), (1.484, 0.037), (-0.182, 0.084), (1.355, -0.057), (1.457, -1.411), (0.010, 0.115), (-0.050, -1.376), (1.477, -1.517), (1.410, -0.017), (0.046, -1.374), (1.403, -1.482), (1.492, 0.149), (0.049, -1.315), (0.005, -0.039), (1.418, -1.555), (0.166, 0.093), (-0.183, -1.531), (0.058, -1.416), (-0.003, -1.560), (0.013, -1.382), (0.064, -1.547), (1.502, -0.106), (1.406, -1.249), (1.533, 0.039), (1.506, -0.115), (0.082, -1.309), (1.297, -1.386), (0.186, -1.304), (1.524, -1.339), (1.284, -1.590), (-0.008, -1.502)}
                \draw[green,fill=green] \Point circle (0.1ex);

            \foreach \Point in {(3.394, 0.114), (3.527, -1.402), (3.396, 0.181), (4.165, 1.142), (2.399, 0.883), (4.377, 1.070), (3.409, -3.339), (2.343, -2.129), (3.479, 0.170), (4.299, 0.934), (3.469, -1.538), (2.396, -2.422), (4.303, -2.283), (2.443, -2.385), (3.320, -3.502), (4.103, 1.182), (3.386, 2.006), (3.223, -0.041), (4.048, 0.815), (4.146, -2.541), (3.210, -0.125), (3.271, 0.123), (4.450, 0.896), (4.329, -2.367), (4.132, 0.872), (4.112, -2.262), (4.247, 1.083), (3.275, -1.441), (3.403, -0.002), (4.444, -2.437), (3.499, 1.917), (4.235, 0.930), (4.196, 1.143), (3.270, 0.128), (3.336, -0.007), (4.265, 0.968), (3.217, -1.500), (2.394, -2.378), (3.368, 0.068), (4.315, -2.393), (4.292, -2.214), (2.417, -2.270), (4.341, 1.071), (3.273, 0.249), (4.285, -2.346), (4.150, -2.241), (2.307, -2.331), (3.412, -1.313), (3.386, -1.653), (3.339, -0.017), (3.394, -0.030), (3.364, -1.390), (3.333, -1.332), (3.358, 0.088), (3.291, -1.241), (2.265, 0.946), (4.102, 0.910), (3.477, -3.303), (3.254, -1.349), (2.278, 0.764), (4.196, 1.174), (2.372, 0.777), (4.371, -2.154), (4.296, -2.365), (3.411, -1.483), (3.381, -1.144), (4.162, 0.982), (3.392, 0.045), (3.438, -1.421), (4.225, -2.356), (3.156, -1.422), (3.573, -0.056), (4.337, -2.298), (3.391, 0.008), (4.185, 1.012), (4.326, -2.333), (2.512, -2.367), (3.231, -0.077), (4.321, 1.080), (4.186, 0.840), (2.247, 1.078), (2.444, -2.270), (4.399, -2.353), (3.366, -1.327), (3.344, 0.210), (3.252, -1.389), (2.395, -2.568), (4.149, -2.333), (3.278, 1.950), (3.253, -0.010), (4.310, -2.577), (3.427, -0.206), (3.376, -3.348), (3.374, -1.244), (3.193, -3.420), (3.272, 0.158), (4.384, 1.047), (3.191, -3.397), (2.404, 1.011), (4.184, 1.129)}
                \draw[red,fill=red] \Point circle (0.1ex);
        \end{tikzpicture}}
    \end{minipage}
    \caption{(a) Two different sets of random walkers starting in $ \{ \mc[1]{v}, \dots, \mc[4]{v} \} $ and $ \{ \mc[7]{v}, \dots, \mc[10]{v} \} $ marked in green and red, respectively. The weight of the solid edges is 1 and the weight of the dashed edges is 0.01. (b) Positions after one step forward. Only one green random walker leaves the set. (c) Positions after one step forward and one step backward. The green set is coherent---only two random walkers escaped---, whereas the red set is clearly less invariant under the forward--backward dynamics.}
    \label{fig:Coherent set illustration}
\end{figure}
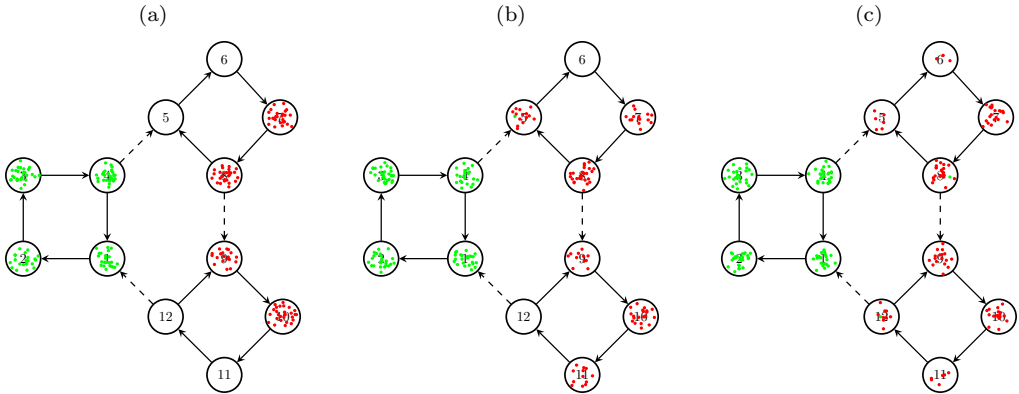

\end{example}

This notion of coherence can be easily extended to time-evolving networks as shown in~\cite{KD22}. The structure of the clusters may then also change in time.

\subsection{Transfer operators}

Transfer operators on graphs can be either directly expressed in terms of graph-related matrices (such as the transition probability matrix) or estimated from random walk data. We will start with a formal definition of transfer operators. In what follows, $ \mathbb{U} = \{ f \colon \mathcal{V} \to \R \} $ denotes the set of all real-valued functions defined on the set of vertices $ \mathcal{V} $ of a graph $ \mc{G} $.

\begin{definition}[Perron--Frobenius \& Koopman operators]
Let $ S $ be the transition probability matrix associated with the graph $ \mc{G} $.
\begin{enumerate}[leftmargin=3.5ex, itemsep=0ex, topsep=0.5ex, label=\roman*)]
\item For a function $ \rho \in \mathbb{U} $, the \emph{Perron--Frobenius operator} $ \mathcal{P} \colon \mathbb{U} \to \mathbb{U} $ is defined by
\begin{equation*}
    \mathcal{P} \rho(\mc[i]{v}) = \sum_{j=1}^n s_{ji} \ts \rho(\mc[j]{v}).
\end{equation*}
\item Similarly, for a function $ f \in \mathbb{U} $, the \emph{Koopman operator} $ \mathcal{K} \colon \mathbb{U} \to \mathbb{U} $ is given by
\begin{equation*}
    \mathcal{K} f(\mc[i]{v}) = \sum_{j=1}^n s_{ij} \ts f(\mc[j]{v}).
\end{equation*}
\end{enumerate}
\end{definition}

The Koopman operator is the adjoint of the Perron--Frobenius operator with respect to the standard inner product.

\begin{definition}[Probability density]
A \emph{probability density} is a function $ \mu \in \mathbb{U} $ that satisfies $ \mu(\mc[i]{v}) \ge 0 $ and
\begin{equation*}
    \sum_{i=1}^n \mu(\mc[i]{v}) = 1.
\end{equation*}
\end{definition}

This allows us to define the Perron--Frobenius operator with respect to different probability densities. We define the $ \mu $-weighted inner product by
\begin{equation*}
    \innerprod{f}{g}_\mu = \sum_{i=1}^n \mu(\mc[i]{v}) \ts f(\mc[i]{v}) \ts g(\mc[i]{v}).
\end{equation*}

\begin{definition}[Reversibility]
The system is called \emph{reversible} if there exists a probability density $ \pi $ that satisfies the \emph{detailed balance condition} $ \pi(\mc[i]{v}) \ts s_{ij} = \pi(\mc[j]{v}) \ts s_{ji} $ for all $ \mc[i]{v}, \mc[j]{v} \in \mc{V} $.
\end{definition}

Random walks on undirected graphs, for instance, are reversible~\cite{Lov93, Aldous14}. We will use this property to analyze the relationships between conventional spectral clustering and our approach.

\begin{definition}[Reweighted Perron--Frobenius operator]
Let $ \mu $ be the initial density and $ \nu = \mathcal{P} \mu $ the resulting image density, i.e.,
\begin{equation*}
    \nu(\mc[i]{v}) = \sum_{j=1}^n s_{ji} \ts \mu(\mc[j]{v}).
\end{equation*}
Furthermore, assume that $ \mu $ and $ \nu $ are strictly positive. Given a function $ u \in \mathbb{U} $, we define the \emph{reweighted Perron--Frobenius operator} $ \mathcal{T} \colon \mathbb{U} \to \mathbb{U} $ by
\begin{equation*}
    \mathcal{T} u(\mc[i]{v}) = \frac{1}{\nu(\mc[i]{v})} \sum_{j=1}^n s_{ji} \ts \mu(\mc[j]{v}) \ts u(\mc[j]{v}).
\end{equation*}
\end{definition}

The reweighted Perron--Frobenius operator describes the evolution of functions with respect to the initial and final densities $ \mu $ and $ \nu $. The assumption that $ \mu(\mc[i]{v}) > 0 $ implies that the probability that a random walker starts in $ \mc[i]{v} $ is non-zero. Correspondingly, $ \nu(\mc[i]{v}) > 0 $ means that $ \mc[i]{v} $ is reachable from at least one vertex, which could be $ \mc[i]{v} $ itself. Adding self-loops thus regularizes the problem.

\begin{definition}[Forward--backward \& backward--forward operators]
Let $ f, u \in \mathbb{U} $ as before.
\begin{enumerate}[leftmargin=3.5ex, itemsep=0ex, topsep=0.5ex, label=\roman*)]
\item We define the \emph{forward--backward operator} $ \mathcal{F} \colon \mathbb{U} \to \mathbb{U} $ by
\begin{equation*}
    \mathcal{F} u(\mc[i]{v})
        = \mathcal{K} \mathcal{T} u(\mc[i]{v})
        = \sum_{j=1}^n s_{ij} \frac{1}{\nu(\mc[j]{v})} \sum_{k=1}^n s_{kj} \ts \mu(\mc[k]{v}) \ts u(\mc[k]{v}).
\end{equation*}
\item Analogously, the \emph{backward--forward operator} $ \mathcal{B} \colon \mathbb{U} \to \mathbb{U} $ is defined by
\begin{equation*}
    \mathcal{B} f(\mc[i]{v})
        = \mathcal{T} \mathcal{K} f(\mc[i]{v})
        = \frac{1}{\nu(\mc[i]{v})} \sum_{j=1}^n s_{ji} \ts \mu(\mc[j]{v}) \sum_{k=1}^n s_{jk} \ts f(\mc[k]{v}).
\end{equation*}
\end{enumerate}
\end{definition}

In order to detect coherent sets, we want to find eigenfunctions $ \varphi_\ell $ of the forward--backward operator $ \mathcal{F} $ corresponding to eigenvalues $ \lambda_\ell \approx 1 $ such that $ \mathcal{F} \varphi_\ell = \lambda_\ell \ts \varphi_\ell \approx \varphi_\ell $. These functions are almost invariant under the forward--backward dynamics. Defining $ \psi_\ell := \frac{1}{\sqrt{\lambda_\ell}} \mathcal{T} \varphi_\ell $, this results in
\begin{equation*}
    \mathcal{K} \psi_\ell = \sqrt{\lambda_\ell} \ts \varphi_\ell, \quad
    \mathcal{T} \varphi_\ell = \sqrt{\lambda_\ell} \ts \psi_\ell, \quad
    \mathcal{F} \varphi_\ell = \lambda_\ell \ts \varphi_\ell, \quad
    \mathcal{B} \psi_\ell = \lambda_\ell \ts \psi_\ell.
\end{equation*}
That is, the functions $ \varphi_\ell $ and $ \psi_\ell $ come in pairs.

\subsection{Covariance and cross-covariance operators}

In addition to these different types of transfer operators, we define (uncentered) covariance and cross-covariance operators.

\begin{definition}[Covariance operators]
Let $ S $ be the transition probability matrix and $ \mu $ and $ \nu $ the densities defined above. Given functions $ f, g \in \mathbb{U} $, we call $ \mathcal{C}_{xx}, \mathcal{C}_{yy} \colon \mathbb{U} \to \mathbb{U} $, with
\begin{equation*}
    \mathcal{C}_{xx} \ts f(\mc[i]{v}) = \mu(\mc[i]{v}) \ts f(\mc[i]{v})
    \quad \text{and} \quad
       \mathcal{C}_{yy} \ts g(\mc[i]{v}) = \nu(\mc[i]{v}) \ts g(\mc[i]{v}),
 \end{equation*}
\emph{covariance operators} and $ \mathcal{C}_{xy} \colon \mathbb{U} \to \mathbb{U} $, with
\begin{equation*}
    \mathcal{C}_{xy} \ts g(\mc[i]{v}) = \sum_{j=1}^n \mu(\mc[i]{v}) \ts s_{ij} \ts g(\mc[j]{v}),
\end{equation*}
\emph{cross-covariance operator}.
\end{definition}

Let $ X \sim \mu $ and $ Y \sim \nu $ be random variables representing the initial position of the random walker and the position after one step, respectively. It follows that
\begin{align*}
    \innerprod{f}{\mathcal{C}_{xx} \ts f} &= \innerprod{f}{f}_\mu = \sum_{i=1}^n \mu(\mc[i]{v}) \ts f(\mc[i]{v})^2 = \mathbb{E}_X\big[f(X)^2\big], \\
     \innerprod{g}{\mathcal{C}_{yy} \ts g} &= \innerprod{g}{g}_\nu = \sum_{i=1}^n \nu(\mc[i]{v}) \ts g(\mc[i]{v})^2 = \mathbb{E}_Y\big[g(Y)^2\big],
\end{align*}
and, since the joint probability distribution is given by $ \mathbb{P}[X = \mc[i]{v}, Y = \mc[j]{v}] = \mu(\mc[i]{v}) \ts s_{ij} $,
\begin{equation*}
    \innerprod{f}{\mathcal{C}_{xy} \ts g} = \innerprod{f}{\mathcal{K} g}_\mu =  \sum_{i=1}^n \sum_{j=1}^n \mu(\mc[i]{v}) \ts s_{ij} \ts f(\mc[i]{v}) \ts g(\mc[j]{v}) = \mathbb{E}_{XY}\big[f(X) \ts g(Y)\big].
\end{equation*}

\subsection{Matrix representations of operators}

In order to simplify the notation, given any function $ g \in \mathbb{U} $, let $ \boldsymbol{g} \in \R^n $ be defined by
\begin{equation*}
    \boldsymbol{g} = \big[\, g(\mc[1]{v}), \dots, g(\mc[n]{v})\, \big]^\top,
\end{equation*}
i.e., $ \boldsymbol{\rho}, \boldsymbol{f}, \boldsymbol{u}, \boldsymbol{\mu}, \boldsymbol{\nu} \in \R^n $ are the vector representations of the functions $ \rho, f, u, \mu, \nu \in \mathbb{U} $, respectively. Additionally, we define diagonal matrices
\begin{equation*}
    D_\mu = \diag(\boldsymbol{\mu})
    \quad \text{and} \quad
    D_\nu = \diag(\boldsymbol{\nu})
\end{equation*}
so that we can write
\begin{align*}
    \mathcal{P} \boldsymbol{\rho} &:= P \ts \boldsymbol{\rho} = S^\top \boldsymbol{\rho}, \\
    \mathcal{K} \boldsymbol{f} &:= K \ts \boldsymbol{f} = S \boldsymbol{f}, \\
    \mathcal{T} \ts \boldsymbol{u} &:= T \ts \boldsymbol{u} = D_\nu^{-1} \ts S^\top D_\mu \ts \boldsymbol{u}, \\
    \mathcal{F} \ts \boldsymbol{u} &:= F \ts \boldsymbol{u} = S \ts D_\nu^{-1} \ts S^\top D_\mu \ts \boldsymbol{u}, \\
    \mathcal{B} \ts \boldsymbol{f} &:= B \ts \boldsymbol{f} = D_\nu^{-1} \ts S^\top D_\mu \ts S \ts \boldsymbol{f},
\end{align*}
where the operators (with a slight abuse of notation) are applied component-wise. That is, $ P, K, T, F, B $ are the matrix representations of the corresponding operators $ \mathcal{P}, \mathcal{K}, \mathcal{T}, \mathcal{F}, \mathcal{B} $, respectively. The matrices $ D_\mu $ and $ D_\nu $ are invertible since we assumed $ \mu $ and $ \nu $ to be strictly positive. For the covariance and cross-covariance operators $ \mathcal{C}_{xx} $, $ \mathcal{C}_{yy} $, and $ \mathcal{C}_{xy} $, we obtain the matrix representations $ C_{xx} = D_\mu $, $ C_{yy} = D_\nu $, and $ C_{xy} = D_\mu \ts S $ and can thus write
\begin{equation*}
    \innerprod{f}{\mathcal{C}_{xx} \ts f} = \mathbf{f}^\top D_\mu \ts \mathbf{f},
    \quad
    \innerprod{g}{\mathcal{C}_{yy} \ts g} = \mathbf{g}^\top D_\nu \ts \mathbf{g},
    \quad \text{and} \quad
    \innerprod{f}{\mathcal{C}_{xy} \ts g} = \mathbf{f}^\top D_\mu S \ts \mathbf{g}.
\end{equation*}
In what follows, we will use functions and their vector representations as well as operators and their matrix representations interchangeably.

\begin{remark}
Let $ \mathds{1} \in \R^n $ be the vector of all ones. In \cite{KD22}, we assumed that the random walkers are initially uniformly distributed and defined $ \widetilde{F} = S \ts \widetilde{D}_\nu^{-1} S^\top $, with $ \widetilde{D}_\nu = \diag(S^\top \mathds{1}) $, whereas setting $ \boldsymbol{\mu} = \frac{1}{n} \mathds{1} $ results in $ D_\nu = \frac{1}{n}\diag(S^\top \mathds{1}) $. However, for this special case $ F = S \ts D_\nu^{-1} S^\top D_\mu = S \ts \widetilde{D}_\nu^{-1} S^\top = \widetilde{F} $ and the definitions are equivalent.
\end{remark}

The properties of transfer operators are analyzed in detail in Appendix~\ref{app:Properties of transfer operators}. We in particular show that the eigenvalues of the operators $ \mathcal{F} $ and $ \mathcal{B} $ are contained in the interval $ [0, 1] $.

\subsection{Associated graph structures}

How are the graph structures associated with the matrix representations $ F $ and $ B $ of the forward--backward and backward--forward operators $ \mathcal{F} $ and $ \mathcal{B} $ related to the original directed graph $ \mc{G} $?

\begin{lemma} \label{lem:graph structures}
It holds that the entry $ f_{ij} $ of the matrix $ F $ is nonzero if and only if there exists an index $ k $ such that $ (\mc[i]{v}, \mc[k]{v}) \in \mc{E} $ and $ (\mc[j]{v}, \mc[k]{v}) \in \mc{E} $. That is, a random walker can go forward from $ \mc[i]{v} $ to $ \mc[k]{v} $ and then backward to $ \mc[j]{v} $. Similarly, the entry $ b_{ij} $ of the matrix $ B $ is nonzero if and only if there exists an index $ k $ such that $ (\mc[k]{v}, \mc[i]{v}) \in \mc{E} $ and $ (\mc[k]{v}, \mc[j]{v}) \in \mc{E} $.
\end{lemma}

\begin{proof}
Since $ s_{ij} \neq 0 ~\Longleftrightarrow~ (\mc[i]{v}, \mc[j]{v}) \in \mc{E} $, we have
\begin{equation*}
    f_{ij} = \sum_{k=1}^n \frac{s_{ik} \ts s_{jk} \ts \mu(\mc[j]{v})}{\nu(\mc[k]{v})} \neq 0 ~\Longleftrightarrow~ \exists k: (\mc[i]{v}, \mc[k]{v}) \in \mc{E} \land (\mc[j]{v}, \mc[k]{v}) \in \mc{E}
\end{equation*}
and
\begin{equation*}
    b_{ij} = \sum_{k=1}^n \frac{s_{ki} \ts s_{kj} \ts \mu(\mc[k]{v})}{\nu(\mc[i]{v})} \neq 0 ~\Longleftrightarrow~ \exists k: (\mc[k]{v}, \mc[i]{v}) \in \mc{E} \land (\mc[k]{v}, \mc[j]{v}) \in \mc{E}. \qedhere
\end{equation*}
\end{proof}

\begin{figure}
    \centering
    \begin{minipage}[t]{0.25\textwidth}
        \centering
        \subfiguretitle{(a)}
        \vspace*{1ex}
        \resizebox{0.6\textwidth}{!}{%
        \begin{tikzpicture}[
                >= stealth, 
                semithick 
            ]
            \tikzstyle{every state}=[
                draw = black,
                thick,
                fill = white,
                inner sep=0pt,
                text width=6mm,
                align=center,
                scale=0.6
            ]

            \node[state] (vk) {k};
            \node[state] (vi) [above left=0.55cm and 0.55cm of vk] {i};
            \node[state] (vj) [above right=0.55cm and 0.55cm of vk] {j};

            \path[->] (vi) edge node {} (vk);
            \path[->] (vj) edge node {} (vk);
        \end{tikzpicture}}
    \end{minipage}
    \begin{minipage}[t]{0.25\textwidth}
        \centering
        \subfiguretitle{(b)}
        \vspace*{1ex}
        \resizebox{0.6\textwidth}{!}{%
        \begin{tikzpicture}[
                >= stealth, 
                semithick 
            ]
            \tikzstyle{every state}=[
                draw = black,
                thick,
                fill = white,
                inner sep=0pt,
                text width=6mm,
                align=center,
                scale=0.6
            ]

            \node[state] (vk) {k};
            \node[state] (vi) [above left=0.55cm and 0.55cm of vk] {i};
            \node[state] (vj) [above right=0.55cm and 0.55cm of vk] {j};

            \path[<-] (vi) edge node {} (vk);
            \path[<-] (vj) edge node {} (vk);
        \end{tikzpicture}}
    \end{minipage}
    \caption{(a) Forward--backward edge between $ \mc[i]{v} $ and $ \mc[j]{v} $. (b) Backward--forward edge between $ \mc[i]{v} $ and $ \mc[j]{v} $.}
    \label{fig:Stucture of Laplacians}
\end{figure}
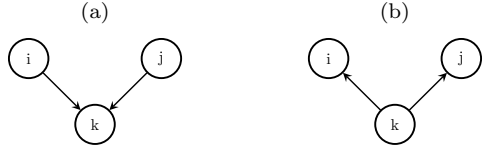

The difference between the two matrix representations of the operators is illustrated in Figure~\ref{fig:Stucture of Laplacians}. Note that the probability of each forward--backward random walk via $ \mc[k]{v} $ is divided by $ \nu(\mc[k]{v}) $, which represents the sum of the probabilities to go from any vertex to $ \mc[k]{v} $. This is intuitively clear since the more edges enter $ \mc[k]{v} $, the more options the random walker has to go backward. On the other hand, the more forward--backward random walks from $ \mc[i]{v} $ to $ \mc[j]{v} $ exist, the larger the weight~$ f_{ij} $.

\begin{remark} \label{rem:DDBS and Hermitian clustering}
This is strongly related to the bibliometric graph clustering approach described in \cite{SaPa11}. Here, the bibliographic coupling matrix $ A A^\top $ and the co-citation matrix $ A^\top A $ represent the common out-links and in-links between a pair of nodes, respectively. Using the sum of these two matrices gives a symmetrization of the adjacency matrix that accounts for both of these types of common links. As in the forward--backward approach, the weights can be divided by the degrees of the nodes in order to give an effective measure of similarity for nodes that accounts for the number of possible routes or links into and out of the node. Alternatively, in \cite{CLSZ20} a complex-valued Hermitian matrix $ H = i(A-A^\top) $ is constructed so that $ H^2 = A A^\top + A^\top A - A^2 - (A^\top)^2 $. This can then also be viewed as a symmetrization scheme that utilizes the bibliographic coupling matrix $ A A^\top $ and the co-citation matrix $ A^\top A $.
\end{remark}

\begin{example}
Let us consider the graph depicted in Figure~\ref{fig:Forward-backward structure}(a), which we already analyzed in Example~\ref{ex:Coherent set illustration}, and compute the associated matrices $ F $ and $ B $. The resulting graph structures are shown in Figures~\ref{fig:Forward-backward structure}(b) and \ref{fig:Forward-backward structure}(c). The graph corresponding to the matrix $ F $, for instance, now contains an edge connecting $ \mc[4]{v} $ and $ \mc[8]{v} $ since a random walker could move forward from $ \mc[4]{v} $ to $ \mc[5]{v} $ and backward to $ \mc[8]{v} $ or the other way around. The probability of this forward--backward random walk, however, is low due to the small edge weight of $ (\mc[4]{v}, \mc[5]{v}) $. Note that without the self-loops all vertices except $ \mc[4]{v} $, $ \mc[8]{v} $, and $ \mc[12]{v} $ would become disconnected. Every forward--backward random walk starting in one of these vertices would then end up in the initial position (since these vertices have only one outgoing edge) and thus form clusters of size one. This illustrates that adding self-loops can be regarded as a form of regularization. \exampleSymbol

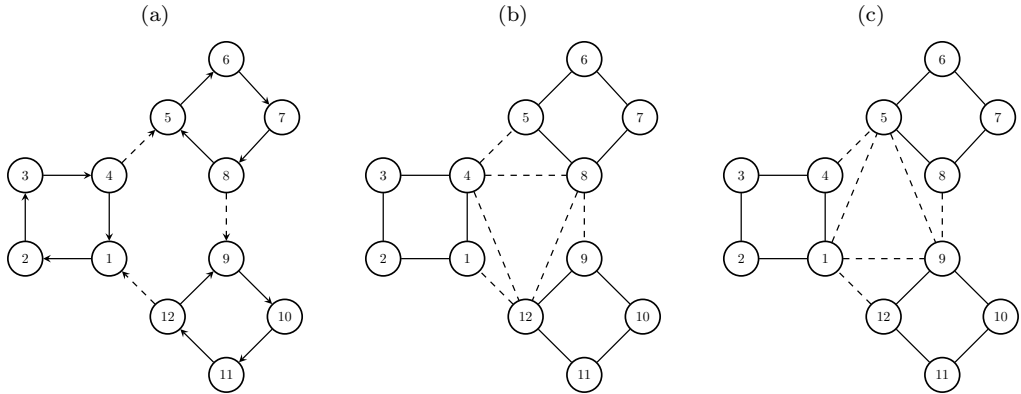
\begin{figure}
    \centering
    \begin{minipage}[t]{0.3\textwidth}
        \centering
        \subfiguretitle{(a)}
        \vspace*{1ex}
        \resizebox{0.85\textwidth}{!}{%
        \begin{tikzpicture}[
                >= stealth, 
                semithick 
            ]
            \tikzstyle{every state}=[
                draw = black,
                thick,
                fill = white,
                inner sep=0pt,
                text width=6mm,
                align=center,
                scale=0.6
            ]

            \node[state] (v3) {3};
            \node[state] (v4) [right=0.8cm of v3] {4};
            \node[state] (v1) [below=0.8cm of v4] {1};
            \node[state] (v2) [below=0.8cm of v3] {2};

            \node[state] (v8) [right=1.35cm of v4] {8};
            \node[state] (v5) [above left=0.55cm and 0.55cm of v8] {5};
            \node[state] (v6) [above=1.35cm of v8] {6};
            \node[state] (v7) [right=1.3cm of v5] {7};

            \node[state] (v9) [right=1.35cm of v1] {9};
            \node[state] (v10) [below right=0.55cm and 0.55cm of v9] {10};
            \node[state] (v11) [below=1.35cm of v9] {11};
            \node[state] (v12) [left=1.35cm of v10] {12};

            \path[->] (v1) edge node {} (v2);
            \path[->] (v2) edge node {} (v3);
            \path[->] (v3) edge node {} (v4);
            \path[->] (v4) edge node {} (v1);

            \path[->] (v5.45) edge node {} (v6.225);
            \path[->] (v6.-45) edge node {} (v7.135);
            \path[->] (v7.-135) edge node {} (v8.45);
            \path[->] (v8.135) edge node {} (v5.-45);

            \path[->] (v9.-45) edge node {} (v10.135);
            \path[->] (v10.-135) edge node {} (v11.45);
            \path[->] (v11.135) edge node {} (v12.-45);
            \path[->] (v12.45) edge node {} (v9.225);

            \path[->, dashed] (v4) edge node [dotted] {} (v5);
            \path[->, dashed] (v8) edge node [dotted] {} (v9);
            \path[->, dashed] (v12) edge node [dotted] {} (v1);
        \end{tikzpicture}}
    \end{minipage}
    \begin{minipage}[t]{0.3\textwidth}
        \centering
        \subfiguretitle{(b)}
        \vspace*{1ex}
        \resizebox{0.85\textwidth}{!}{%
        \begin{tikzpicture}[
                >= stealth, 
                semithick 
            ]
            \tikzstyle{every state}=[
                draw = black,
                thick,
                fill = white,
                inner sep=0pt,
                text width=6mm,
                align=center,
                scale=0.6
            ]

            \node[state] (v3) {3};
            \node[state] (v4) [right=0.8cm of v3] {4};
            \node[state] (v1) [below=0.8cm of v4] {1};
            \node[state] (v2) [below=0.8cm of v3] {2};

            \node[state] (v8) [right=1.35cm of v4] {8};
            \node[state] (v5) [above left=0.55cm and 0.55cm of v8] {5};
            \node[state] (v6) [above=1.35cm of v8] {6};
            \node[state] (v7) [right=1.3cm of v5] {7};

            \node[state] (v9) [right=1.35cm of v1] {9};
            \node[state] (v10) [below right=0.55cm and 0.55cm of v9] {10};
            \node[state] (v11) [below=1.35cm of v9] {11};
            \node[state] (v12) [left=1.35cm of v10] {12};

            \path[-] (v1) edge node {} (v2);
            \path[-] (v2) edge node {} (v3);
            \path[-] (v3) edge node {} (v4);
            \path[-] (v4) edge node {} (v1);

            \path[-] (v5.45) edge node {} (v6.225);
            \path[-] (v6.-45) edge node {} (v7.135);
            \path[-] (v7.-135) edge node {} (v8.45);
            \path[-] (v8.135) edge node {} (v5.-45);

            \path[-] (v9.-45) edge node {} (v10.135);
            \path[-] (v10.-135) edge node {} (v11.45);
            \path[-] (v11.135) edge node {} (v12.-45);
            \path[-] (v12.45) edge node {} (v9.225);

            \path[-, dashed] (v4) edge node [dotted] {} (v5);
            \path[-, dashed] (v8) edge node [dotted] {} (v9);
            \path[-, dashed] (v12) edge node [dotted] {} (v1);

            \path[-, dashed] (v4) edge node [dotted] {} (v8);
            \path[-, dashed] (v4) edge node [dotted] {} (v12);
            \path[-, dashed] (v8) edge node [dotted] {} (v12);
        \end{tikzpicture}}
    \end{minipage}
    \begin{minipage}[t]{0.3\textwidth}
        \centering
        \subfiguretitle{(c)}
        \vspace*{1ex}
        \resizebox{0.85\textwidth}{!}{%
        \begin{tikzpicture}[
                >= stealth, 
                semithick 
            ]
            \tikzstyle{every state}=[
                draw = black,
                thick,
                fill = white,
                inner sep=0pt,
                text width=6mm,
                align=center,
                scale=0.6
            ]

            \node[state] (v3) {3};
            \node[state] (v4) [right=0.8cm of v3] {4};
            \node[state] (v1) [below=0.8cm of v4] {1};
            \node[state] (v2) [below=0.8cm of v3] {2};

            \node[state] (v8) [right=1.35cm of v4] {8};
            \node[state] (v5) [above left=0.55cm and 0.55cm of v8] {5};
            \node[state] (v6) [above=1.35cm of v8] {6};
            \node[state] (v7) [right=1.3cm of v5] {7};

            \node[state] (v9) [right=1.35cm of v1] {9};
            \node[state] (v10) [below right=0.55cm and 0.55cm of v9] {10};
            \node[state] (v11) [below=1.35cm of v9] {11};
            \node[state] (v12) [left=1.35cm of v10] {12};

            \path[-] (v1) edge node {} (v2);
            \path[-] (v2) edge node {} (v3);
            \path[-] (v3) edge node {} (v4);
            \path[-] (v4) edge node {} (v1);

            \path[-] (v5.45) edge node {} (v6.225);
            \path[-] (v6.-45) edge node {} (v7.135);
            \path[-] (v7.-135) edge node {} (v8.45);
            \path[-] (v8.135) edge node {} (v5.-45);

            \path[-] (v9.-45) edge node {} (v10.135);
            \path[-] (v10.-135) edge node {} (v11.45);
            \path[-] (v11.135) edge node {} (v12.-45);
            \path[-] (v12.45) edge node {} (v9.225);

            \path[-, dashed] (v4) edge node [dotted] {} (v5);
            \path[-, dashed] (v8) edge node [dotted] {} (v9);
            \path[-, dashed] (v12) edge node [dotted] {} (v1);

            \path[-, dashed] (v1) edge node [dotted] {} (v5);
            \path[-, dashed] (v1) edge node [dotted] {} (v9);
            \path[-, dashed] (v5) edge node [dotted] {} (v9);
        \end{tikzpicture}}
    \end{minipage}
    \caption{(a) Directed graph with three clusters. Self-loops are omitted for the sake of clarity. (b) Graph structure of the matrix~$ F $ for uniform $ \mu $. (c) Graph structure of the matrix~$ B $ for uniform~$ \nu $. The dashed edges have again a low weight. Clustering the dominant eigenvectors of the matrices $ F $ or $ B $ results in three coherent sets, see also Example~\ref{ex:Coherent set illustration}.}
    \label{fig:Forward-backward structure}
\end{figure}

\end{example}

\subsection{Spectral clustering of directed graphs}

Spectral clustering methods for undirected graphs are based on the eigenvectors corresponding to the smallest eigenvalues of an associated graph Laplacian. The \emph{random-walk normalized Laplacian} is defined by
\begin{equation*}
    L_\mathcal{K} = I - D_\mathscr{o}^{-1} A = I - K,
\end{equation*}
where $ K $ is the matrix representation of the Koopman operator. Equivalently, we can also compute the eigenvectors corresponding to the largest eigenvalues of $ K $. This shows that conventional spectral clustering computes the dominant eigenfunctions of the Koopman operator, which contain information about the \emph{metastable sets} of the system. The eigenvalues of the Koopman operator associated with directed graphs, however, are in general complex-valued and standard spectral clustering techniques typically fail, see~\cite{KD22} for a detailed analysis. For directed (and also time-evolving) graphs we thus propose to compute the eigenvectors associated with the largest eigenvalues of the forward--backward or backward--forward operators in order to detect \emph{coherent sets}. The eigenvalues of these operators are real-valued even if the graph is directed, see Appendix~\ref{app:Properties of transfer operators}.

\begin{textalgorithm}[Transfer operator-based spectral clustering algorithm] \label{alg:directed} $ $
\begin{enumerate}[topsep=1ex] \setlength{\itemsep}{0ex}
\item Compute the $ k $ largest eigenvalues $ \lambda_\ell $ and associated eigenvectors $ \boldsymbol{\varphi}_\ell $ of $ F $.
\item Define $ \boldsymbol{\Phi} = [\boldsymbol{\varphi}_1, \dots, \boldsymbol{\varphi}_k] \in \R^{n \times k} $ and let $ r_i $ denote the $ i $th row of $ \boldsymbol{\Phi} $.
\item Cluster the points $ \{ \ts r_i \ts \}_{i=1}^n $ using, e.g., $ k $-means.
\end{enumerate}
\end{textalgorithm}

The number of eigenvalues close to $ 1 $ indicates the number of coherent sets. That is, $ k $ should be chosen in such a way that there is a spectral gap between $ \lambda_k $ and $ \lambda_{k+1} $. In order to illustrate the spectral clustering algorithm, we first consider a simple graph that has a clearly defined cluster structure.

\begin{example} \label{ex:DSBM}

\begin{figure}
    \definecolor{matlab1}{RGB}{0, 114, 189}
    \definecolor{matlab2}{RGB}{217, 83, 25}
    \definecolor{matlab3}{RGB}{237, 177, 32}
    \definecolor{matlab4}{RGB}{126, 47, 142}
    \newcommand{\cdash}[1]{\textcolor{#1}{\rule[0.5ex]{1em}{0.3ex}}}
    \centering
    \begin{minipage}[t]{0.45\textwidth}
        \centering
        \subfiguretitle{(a)}
        \includegraphics[height=0.25\textheight]{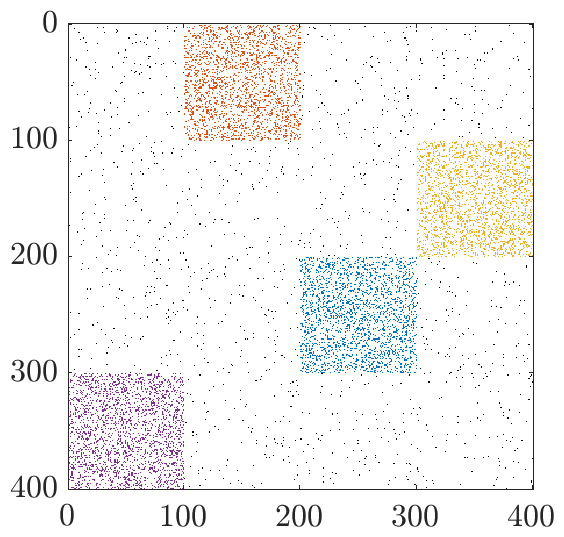}
    \end{minipage}
    \begin{minipage}[t]{0.45\textwidth}
        \centering
        \subfiguretitle{(b)}
        \includegraphics[height=0.25\textheight]{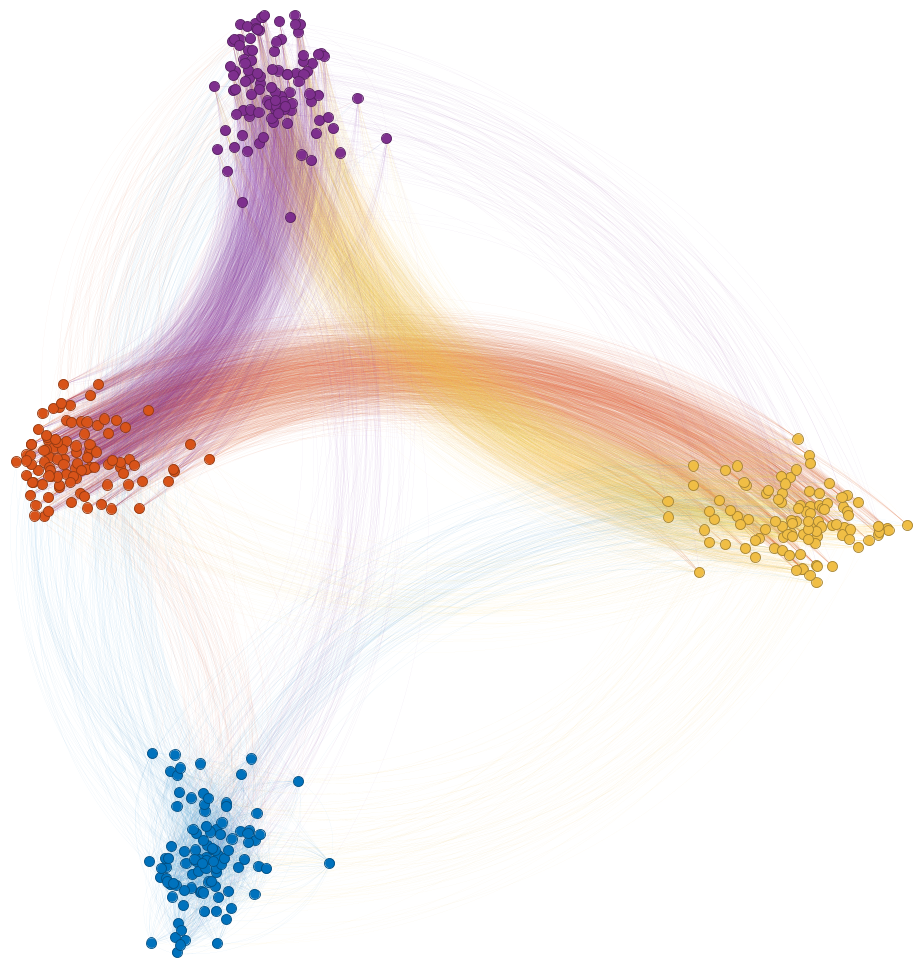}
    \end{minipage} \\[1ex]
    \begin{minipage}[t]{0.45\textwidth}
        \centering
        \subfiguretitle{(c)}
        ~~\includegraphics[height=0.255\textheight]{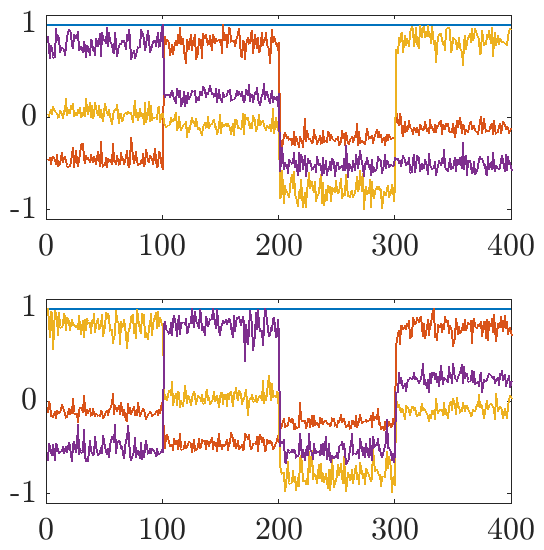}
    \end{minipage}
    \begin{minipage}[t]{0.45\textwidth}
        \centering
        \subfiguretitle{(d)}
        \includegraphics[height=0.255\textheight]{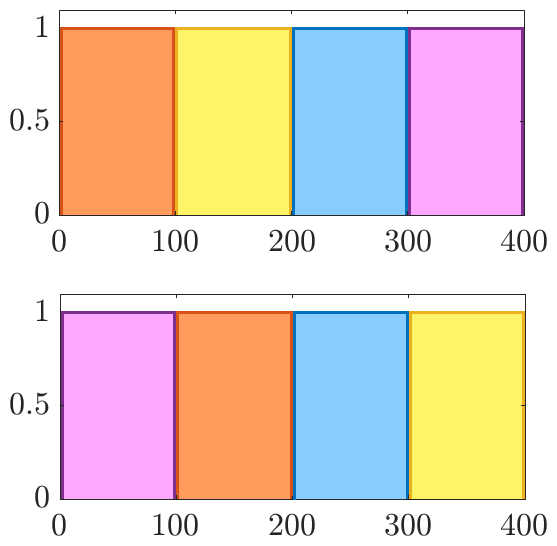}
    \end{minipage}
    \caption{(a) Adjacency matrix of a directed graph comprising four clusters. The blocks are colored according to the cluster numbers. (b) Visualization of the clustered graph using the eigenfunctions $ \varphi_2 $ and $ \varphi_3 $ as coordinates. (c)~Four dominant eigenfunctions $ \varphi_\ell $ (top) and $ \psi_\ell $ (bottom), where \cdash{matlab1} denotes the first, \cdash{matlab2} the second, \cdash{matlab3} the third, and \cdash{matlab4} the fourth eigenfunction. The functions are roughly constant within the clusters. (d) Clusters extracted from the functions $ \varphi_\ell $ (top) and $ \psi_\ell $ (bottom) using $ k $-means. The blue indicator function picks row 3 and column 3 of the block matrix, the red function row 1 and column~2, the yellow function row 2 and column 4, and the purple function row 4 and column 1. That is, each function pair is associated with the block of the adjacency matrix marked in the corresponding color. These blocks all have in common that they contain many nonzero entries (i.e., incoming or outgoing edges). Note that since we compute two functions, i.e., $ \varphi_\ell $ and $ \psi_\ell $, it is now possible to detect dense off-diagonal blocks.}
    \label{fig:DSBM_example}
\end{figure}

We sample a benchmark problem $ \mc{G} \in \mathbf{G}(4, 100, E) $ from the DSBM, with
\begin{equation*}
    E =
    \begin{bmatrix}
        q & p & q & q \\
        q & q & q & p \\
        q & q & p & q \\
        p & q & q & q
    \end{bmatrix},
\end{equation*}
where $ p = 0.8 $ and $ q = 0.1 $. The resulting adjacency matrix is shown in Figure~\ref{fig:DSBM_example}(a). Note that there are two different types of clusters, namely dense blocks on the diagonal, representing groups of vertices that are tightly coupled to other vertices contained in the same cluster, and dense off-diagonal blocks, which have the property that the contained vertices are all tightly coupled (unidirectionally) to vertices in another cluster. For both types of clusters, however, the probability that a random walker going forward and then backward ends up in the same cluster is large compared to the escape probability. Applying Algorithm~\ref{alg:directed} yields the results shown in Figure~\ref{fig:DSBM_example}(c) and (d). The dominant eigenvalues are $ \lambda_1 = 1 $, $ \lambda_2 = 0.72 $, $\lambda_3 = 0.70 $, and $ \lambda_4 = 0.69 $, followed by a spectral gap. The eigenfunctions of the forward--backward operator can also be used for visualizing directed graphs\footnote{Similar visualization techniques for undirected graphs using spectral properties can be found in \cite{Ko05}.} as shown in Figure~\ref{fig:DSBM_example}(b). By defining the position of the vertex $ \mc[i]{v} $ to be $ (\varphi_2(\mc[i]{v}), \varphi_3(\mc[i]{v})) $, we obtain a two-dimensional embedding of the graph. The edges are colored according to the source nodes. The blue vertices form a cluster in the conventional sense since they are strongly connected but only weakly coupled to the other clusters. The remaining sets of vertices have the property that many directed edges are pointing to one of the other clusters but not vice versa. These clusters correspond to non-sparse off-diagonal blocks. \exampleSymbol
\end{example}

\subsection{Spectral clustering of undirected graphs as a special case}

A natural question now is whether conventional spectral clustering methods for undirected graphs can be regarded as a special case of the algorithm derived above.

\begin{lemma} \label{lem:reversible case}
Assume that the system is reversible with respect to $ \pi $. Then, setting $ \mu = \pi $, it holds that $ \mathcal{T} = \mathcal{K} $ and $ \mathcal{F} = \mathcal{B} = \mathcal{K}^2 $.
\end{lemma}

\begin{proof}
First, we rewrite the detailed balance condition as $ D_\pi \ts S = S^\top D_\pi $. Since $ \mu = \pi $ implies $ \nu = \pi $, we obtain $ T = D_\pi^{-1} \ts S^\top D_\pi = S = K $. The second part follows immediately from the definitions of the operators $ \mathcal{F} $ and $ \mathcal{B} $.
\end{proof}

If $ \pi $ satisfies the detailed balance condition, then it is also automatically an invariant density. This can be seen as follows:
\begin{equation*}
    \mathcal{P} \boldsymbol{\pi} = S^\top \ts \boldsymbol{\pi} = S^\top D_\pi \ts \mathds{1} = D_\pi \ts S \ts \mathds{1} = D_\pi \ts \mathds{1} = \boldsymbol{\pi}.
\end{equation*}

\begin{lemma} \label{lem:equivalence for undirected graphs}
Let $ \mc{G} $ now be an undirected graph with unique invariant density $ \pi $. Assume that $ K $ has only distinct positive eigenvalues. Then conventional spectral clustering and Algorithm~\ref{alg:directed} with $ \mu = \pi $ compute the same clusters.
\end{lemma}

\begin{proof}
Let $ \kappa_\ell $ be the eigenvalues of $ K $ and $ \lambda_\ell $ the eigenvalues of $ F $. Using Lemma~\ref{lem:reversible case}, we know that $ \lambda_\ell = \kappa_\ell^2 $ and the eigenvectors are identical. Since all eigenvalues of $ K $ are by assumption positive and distinct, the ordering of the eigenvalues and eigenvectors remains unchanged.
\end{proof}

Negative eigenvalues of $ K $ might affect the ordering of the eigenvectors and thus the spectral clustering results. However, positive eigenvalues can be enforced by turning the transition probability matrix into a lazy Markov chain \cite{LP17}.

\section{Analysis and approximation of transfer operators}
\label{sec:Analysis and approximation}

In this section, we will derive reduced representations of transfer operators, show how they can be estimated from random walk data, and present alternative interpretations of the proposed spectral clustering approach.

\subsection{Galerkin approximation}

Given an operator $ \mathcal{L} \colon \mathbb{U} \to \mathbb{U} $ with matrix representation $ L $, our goal now is to compute its Galerkin projection $ \mathcal{L}_r $ onto the $ r $-dimensional subspace $ \mathbb{U}_r \subset \mathbb{U} $ with basis $ \{ \phi_i \}_{i=1}^r $. Here, $ r $ could potentially be much smaller than $ n $. The matrix representation $ L_r \in \R^{r \times r} $ of $ \mathcal{L}_r $ is given by $ L_r = \big(G^{(0)}\big)^{-1} G^{(1)} $, with entries
\begin{equation*}
    g_{ij}^{(0)} = \innerprod{\phi_i}{\phi_j}_\mu
    \quad \text{and} \quad
    g_{ij}^{(1)} = \innerprod{\phi_i}{\mathcal{L} \phi_j}_\mu.
\end{equation*}
By defining the vector-valued function $ \phi(\mc[i]{v}) = [\phi_1(\mc[i]{v}), \dots, \phi_r(\mc[i]{v})]^\top \in \R^r $ and
\begin{equation*}
    \Phi_{\mc{V}} =
    \begin{bmatrix}
        \phi(\mc[1]{v}), \dots, \phi(\mc[n]{v})
    \end{bmatrix}  \in \R^{r \times n},
\end{equation*}
the matrices $ G^{(0)} $ and $ G^{(1)} $ can be compactly written as
\begin{align*}
    G^{(0)} &= \sum_{i=1}^n \mu(\mc[i]{v}) \ts \phi(\mc[i]{v}) \ts \phi(\mc[i]{v})^\top = \Phi_{\mc{V}} D_\mu \Phi_{\mc{V}}^\top
\intertext{and}
    G^{(1)} &= \sum_{i=1}^n \sum_{j=1}^n \mu(\mc[i]{v}) \ts l_{ij} \ts \phi(\mc[i]{v}) \ts \phi(\mc[j]{v})^\top = \Phi_{\mc{V}} D_\mu \ts L \ts \Phi_{\mc{V}}^\top.
\end{align*}
Any function $ f \in \mathbb{U}_r $ is given by
\begin{equation*}
    f(\mc[i]{v}) = \sum_{j=1}^r c_j \ts \phi_j(\mc[i]{v}) = c^\top \phi(\mc[i]{v}),
\end{equation*}
where $ c = [c_1, \dots, c_r]^\top \in \R^r $, such that
\begin{equation*}
    \mathcal{L}_r f(\mc[i]{v}) = (L_r \ts c)^\top \phi(\mc[i]{v}).
\end{equation*}
Suppose now $ \xi_\ell $ is an eigenvector of $ L_r $ with associated eigenvalue $ \lambda_\ell $, i.e., $ L_r \ts \xi_\ell = \lambda_\ell \ts \xi_\ell $, then $ \varphi_\ell(\mc[i]{v}) = \xi_\ell^\top \phi(\mc[i]{v}) $ is an eigenfunction of $ \mathcal{L}_r $ since
\begin{equation*}
    \mathcal{L}_r \varphi_\ell(\mc[i]{v}) = (L_r \ts \xi_\ell)^\top \phi(\mc[i]{v}) = \lambda_\ell \ts \xi_\ell^\top \phi(\mc[i]{v}) = \lambda_\ell \ts \varphi_\ell(\mc[i]{v}).
\end{equation*}
Alternatively, assuming the matrix $ G^{(0)} $ is invertible, we can also solve the generalized eigenvalue problem $ G^{(1)} \ts \xi_\ell = \lambda_\ell \ts G^{(0)} \ts \xi_\ell $. This shows that eigenfunctions of $ \mathcal{L} $ can be approximated by eigenfunctions of $ \mathcal{L}_r $.

\begin{example} \label{ex:indicator function basis}
Choosing $ r = n $ and indicator functions for each vertex, i.e.,
\begin{equation*}
    \phi_j(\mc[i]{v}) = \mathds{1}_{\mc[j]{v}}(\mc[i]{v}) =
    \begin{cases}
        1, & i = j, \\
        0, & \text{otherwise},
    \end{cases}
\end{equation*}
we obtain $ \Phi_{\mc{V}} = I $ and hence $ G^{(0)} = D_\mu $ and $ G^{(1)} = D_\mu L $ so that $ L_r = L $. That is, in this case we obtain the full matrix representation of the operator $ \mathcal{L} $. \exampleSymbol
\end{example}

Theoretically, we could choose any set of basis functions. Our goal, however, is to define the basis functions in such a way that the reduced operator $ \mathcal{L}_r $ has essentially the same dominant eigenvalues as the full operator $ \mathcal{L} $ and thus retains the metastability (and cluster structure) of the system.

\begin{remark}
The optimal basis is given by the eigenvectors $ \boldsymbol{\varphi}_\ell $ corresponding to the largest eigenvalues $ \lambda_\ell $ of $ L $ since choosing
\begin{equation*}
    \Phi_{\mc{V}} =
    \begin{bmatrix}
        \boldsymbol{\varphi}_1^\top \; \\
        \vdots \\
        \boldsymbol{\varphi}_r^\top
    \end{bmatrix}
\end{equation*}
results in $ G^{(0)} = \Phi_{\mc{V}} \ts D_\mu \ts \Phi_{\mc{V}}^\top $ and $ G^{(1)} = \Phi_{\mc{V}} \ts D_\mu \ts L \ts \Phi_{\mc{V}}^\top = \Phi_{\mc{V}} \ts D_\mu \ts \Phi_{\mc{V}}^\top \diag(\lambda_1, \dots, \lambda_r) $. That is, $ L_r = \diag(\lambda_1, \dots, \lambda_r) $ and the reduced operator has exactly the same dominant eigenvalues as the full operator. However, if we already knew the eigenvectors of $ L $, then we could immediately compute the clusters by applying, e.g., $ k $-means. Also, the number of metastable sets is in general not known in advance. The question then is whether it is possible to derive near-optimal basis functions without having to compute the eigenvectors in the first place. One possibility would be to construct basis functions based on random walks started in different parts of the graph since the random walkers would with a high probability first explore nodes within the clusters before moving to other clusters. The generation of suitable basis functions, however, is beyond the scope of this paper.
\end{remark}

\subsection{Data-driven approximation}

The operators introduced above can also be estimated from random walk data. Given $ m $ random walkers $ x^{(i)} $ sampled from the distribution~$ \mu $, each random walker takes a single step forward according to the probabilities given by the transition probability matrix $ S $. Let the final positions of the random walkers be denoted by $ y^{(i)} $. Alternatively, instead of considering many random walkers, the data can be extracted from one long random walk $ x^{(1)}, \dots, x^{(m+1)} $ by defining $ y^{(i)} = x^{(i+1)} $, for $ i = 1, \dots, m $. The density $ \mu $ is then determined by the random walk. If the graph admits a unique invariant density $ \pi $, then $ \mu $ will converge to $ \pi $ for $ m \to \infty $.

We now define data matrices $ \Phi_x, \Phi_y \in \R^{r \times m} $ by
\begin{equation*}
    \Phi_x = \begin{bmatrix} \phi(x^{(1)}) & \phi(x^{(2)}) & \dots & \phi(x^{(m)}) \end{bmatrix}
    \quad \text{and} \quad
    \Phi_y = \begin{bmatrix} \phi(y^{(1)}) & \phi(y^{(2)}) & \dots & \phi(y^{(m)}) \end{bmatrix}.
\end{equation*}
Further, let $ \widehat{G}_{xx}, \widehat{G}_{yy}, \widehat{G}_{xy} \in \R^{r \times r} $ be defined by
\begin{alignat*}{4}
    \widehat{G}_{xx} &:= \frac{1}{m} \Phi_x \ts \Phi_x^\top = \frac{1}{m} \sum_{i=1}^m \phi\big(x^{(i)}\big) \ts \phi\big(x^{(i)}\big)^\top &&\underset{\scriptscriptstyle m \rightarrow \infty}{\longrightarrow} \sum_{k=1}^n \mu(\mc[k]{v}) \ts \phi(\mc[k]{v}) \ts \phi(\mc[k]{v})^\top =: G_{xx}, \\
    \widehat{G}_{yy} &:= \frac{1}{m} \Phi_y \ts \Phi_y^\top = \frac{1}{m} \sum_{i=1}^m \phi\big(y^{(i)}\big) \ts \phi\big(y^{(i)}\big)^\top &&\underset{\scriptscriptstyle m \rightarrow \infty}{\longrightarrow} \sum_{k=1}^n \nu(\mc[k]{v}) \ts \phi(\mc[k]{v}) \ts \phi(\mc[k]{v})^\top  =: G_{yy}, \\
    \widehat{G}_{xy} &:= \frac{1}{m} \Phi_x \ts \Phi_y^\top = \frac{1}{m} \sum_{i=1}^m \phi\big(x^{(i)}\big) \ts \phi\big(y^{(i)}\big)^\top &&\underset{\scriptscriptstyle m \rightarrow \infty}{\longrightarrow} \sum_{k=1}^n \sum_{l=1}^n \mu(\mc[k]{v}) s_{kl} \ts \phi(\mc[k]{v}) \ts \phi(\mc[l]{v})^\top  =: G_{xy}.
\end{alignat*}
Thus, the entries of the matrices $ G_{xx} = \Phi_{\mc{V}} \ts D_\mu \ts \Phi_{\mc{V}}^\top $, $ G_{yy} = \Phi_{\mc{V}} \ts D_\nu \ts \Phi_{\mc{V}}^\top $, and $ G_{xy} = \Phi_{\mc{V}} \ts D_\mu \ts S \ts \Phi_{\mc{V}}^\top $ are given by
\begin{align*}
    [g_{xx}]_{\rule{0pt}{1.5ex}ij} = \innerprod{\phi_i}{\phi_j}_\mu, \quad
    [g_{yy}]_{\rule{0pt}{1.5ex}ij} = \innerprod{\phi_i}{\phi_j}_\nu, \quad
    [g_{xy}]_{\rule{0pt}{1.5ex}ij} = \innerprod{\phi_i}{\mathcal{K} \phi_j}_\mu.
\end{align*}

\begin{remark}
Note that the matrices $ G_{xx} $ and $ G_{xy} $ are not the Gram matrices used in kernel-based methods. Here, $ G $ stands for Galerkin. For functions $ f(\mc[i]{v}) = c_x^\top \phi(\mc[i]{v}) $ and $ g(\mc[i]{v}) = c_y^\top \phi(\mc[i]{v}) $, it holds that $ \innerprod{f}{\mathcal{C}_{xx} f} = c_x^\top \ts G_{xx} \ts c_x $ and $ \innerprod{f}{\mathcal{C}_{xy} g} = c_x^\top \ts G_{xy} \ts c_y $. If we choose the basis functions defined in Example~\ref{ex:indicator function basis}, then $ G_{xx} = C_{xx} = D_\mu $ and $ G_{xy} = C_{xy} = D_\mu \ts S $.
\end{remark}

\begin{lemma} \label{lem:Galerkin approximations}
It follows that:
\begin{enumerate}[leftmargin=3.5ex, itemsep=0ex, topsep=0.5ex, label=\roman*)]
\item $ G_{xx}^{-1} \ts G_{xy} $ is a Galerkin approximation of the operator $ \mathcal{K} $ w.r.t.\ $ \innerprod{\cdot}{\cdot}_\mu $,
\item $ G_{yy}^{-1} \ts G_{yx} $ is a Galerkin approximation of the operator $ \mathcal{T} $ w.r.t.\ $ \innerprod{\cdot}{\cdot}_\nu $.
\end{enumerate}
\end{lemma}

\begin{proof}
For the first part, compare $ G_{xx} $ and $ G_{xy} $ with the matrices $ G^{(0)} $ and $ G^{(1)} $ introduced above. Then, using Proposition~\ref{pro:operator properties}, we have
\begin{equation*}
    [g_{xy}]_{\rule{0pt}{1.5ex}ij} = \innerprod{\mathcal{T} \phi_i}{\phi_j}_\nu ~~ \Longrightarrow ~~ [g_{yx}]_{\rule{0pt}{1.5ex}ij} =  \innerprod{\phi_i}{\mathcal{T} \phi_j}_\nu,
\end{equation*}
which proves the second part.
\end{proof}

If $ \mu $ is the uniform distribution, then $ \innerprod{\phi_i}{\mathcal{K} \phi_j} = \innerprod{\mathcal{P} \phi_i}{\phi_j} $, and $ G_{xx}^{-1} \ts G_{yx} $ is an approximation of the operator $ \mathcal{P} $ with respect to the standard inner product $ \innerprod{\cdot}{\cdot} $. As the forward--backward and backward--forward operators are simply compositions of these operators (see also Proposition~\ref{pro:covariance operator representations}), all the transfer operators introduced above can be estimated from data. This, however, implies that we do not directly compute Galerkin approximations of $ \mathcal{F} $ and $ \mathcal{B} $ but represent them as compositions of Galerkin projections, which introduces additional errors.

\begin{example} \label{ex:DSBM data-driven}
For the graph introduced in Example~\ref{ex:DSBM}, we define sets $ \mathbb{A}_1 = \{ \mc[1]{v}, \dots, \mc[100]{v} \} $, $ \mathbb{A}_2 = \{ \mc[101]{v}, \dots, \mc[200]{v}\}$, $ \mathbb{A}_3 = \{ \mc[201]{v}, \dots, \mc[300]{v} \} $, and $ \mathbb{A}_4 = \{ \mc[301]{v}, \dots, \mc[400]{v} \} $ and basis functions $ \phi_j(\mc[i]{v}) = \mathds{1}_{\mathbb{A}_j}(\mc[i]{v}) $, with $ j = 1, \dots, 4 $, where $ \mathds{1}_\mathbb{B} $ denotes the indicator function for the set $ \mathbb{B} $. The representation of the forward--backward operator $ \mathcal{F} $ projected onto the space spanned by the four basis functions is then approximately a diagonal matrix and can be regarded as a coarse-grained counterpart of the full operator, where the off-diagonal entries correspond to transition probabilities between the clusters. The resulting eigenvalues and eigenvectors are good approximations of the dominant eigenvalues and eigenvectors of the full operator as shown in Figure~\ref{fig:DSBM data-driven}. \exampleSymbol

\begin{figure}
    \definecolor{matlab1}{RGB}{0, 114, 189}
    \definecolor{matlab2}{RGB}{217, 83, 25}
    \definecolor{matlab3}{RGB}{237, 177, 32}
    \definecolor{matlab4}{RGB}{126, 47, 142}
    \newcommand{\cdash}[1]{\textcolor{#1}{\rule[0.5ex]{1em}{0.3ex}}}
    \centering
    \begin{minipage}[t]{0.45\textwidth}
        \centering
        \subfiguretitle{(a)}
        \includegraphics[height=0.27\textheight]{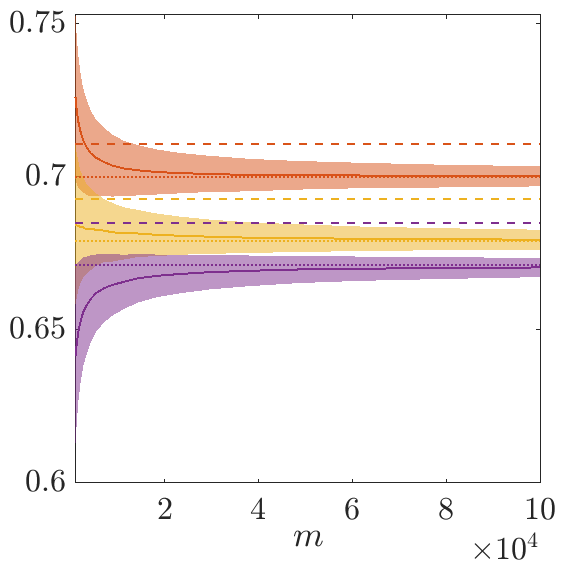}
    \end{minipage}
    \begin{minipage}[t]{0.45\textwidth}
        \centering
        \subfiguretitle{(b)}
        \includegraphics[height=0.25\textheight]{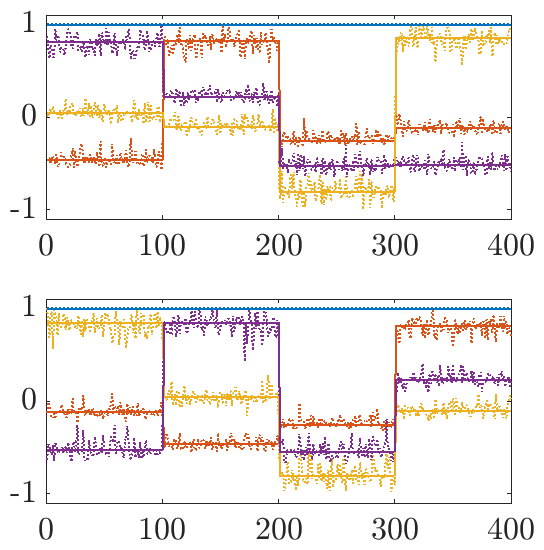}
    \end{minipage}
    \caption{(a) Mean of the eigenvalues estimated from random walk data (averaged over 1000 runs) as a function of $ m $, where \cdash{matlab1} denotes the first (not shown since it is always 1), \cdash{matlab2} the second, \cdash{matlab3} the third, and \cdash{matlab4} the fourth eigenvalue. The shaded area in the corresponding color represents the standard deviation, the dotted line the infinite-data limit, and the dashed line the eigenvalues of the Galerkin projection of the operator. Recall that the data-driven approximation is a composition of two Galerkin approximations and thus slightly underestimates the correct eigenvalues. The eigenvalues of the Galerkin approximation are also marginally smaller than the eigenvalues of the full operator computed in Example~\ref{ex:DSBM}. (b) Dominant eigenfunctions  $ \varphi_\ell $ (top) and $ \psi_\ell$ (bottom) of the reduced matrix (solid lines) and the full matrix (dotted lines). The chosen indicator function basis approximates the correct eigenfunctions well since they are essentially almost constant within the clusters.}
    \label{fig:DSBM data-driven}
\end{figure}

\end{example}

The example illustrates that it is possible to approximate the dominant eigenvalues and eigenvectors by computing eigendecompositions of potentially much smaller matrices, provided that a suitable set of basis functions is chosen and that the graph exhibits metastability.

\begin{remark}
Assuming that only random-walk data is available, but not the underlying graph itself, the data-driven approach can thus also be used for inferring the network structure and the transition probabilities. This will be further investigated in future work.
\end{remark}

\subsection{Optimization problem formulation}

We will now show that it is possible to rewrite the clustering approach for directed graphs as a constrained optimization problem. This gives rise to a different interpretation of clusters. \emph{Canonical correlation analysis} (CCA) \cite{Hotelling36, MRB01, ShCh04} was originally developed to maximize the correlation between random variables---potentially in infinite-dimensional feature spaces if the kernel-based formulation is used. It has been shown in \cite{KHMN19} that CCA, when applied to time-series data, can be used to detect coherent sets. In order to illustrate the relationships with the forward--backward operator $ \mathcal{F} $ (or its matrix representation $ F $), we will briefly outline the derivation of CCA, which can be written as
\begin{equation*}
    \max_{f, g \in \mathbb{U}} \innerprod{f}{\mathcal{C}_{xy} \ts g} \quad \text{s.t.} \quad \innerprod{f}{\mathcal{C}_{xx} \ts f} = \innerprod{g}{\mathcal{C}_{yy} \ts g} = 1.
\end{equation*}
Restricted to the subspace $ \mathbb{U}_r $, defining $ f(\mc[i]{v}) = c_x^\top \phi(\mc[i]{v}) $ and $ g(\mc[i]{v}) = c_y^\top \phi(\mc[i]{v}) $, this yields
\begin{equation} \label{eq:CCA}
    \max_{c_x, c_y \in \R^r} c_x^\top G_{xy} \ts c_y \quad \text{s.t.} \quad c_x^\top G_{xx} \ts c_x = c_y^\top G_{yy} \ts c_y = 1.
\end{equation}
The Lagrangian function corresponding to this optimization problem is
\begin{equation*}
    \mathfrak{L}(c_x, c_y, \kappa_x, \kappa_y) = c_x^\top G_{xy} \ts c_y - \frac{\kappa_x}{2} \big(c_x^\top G_{xx} \ts c_x - 1\big) - \frac{\kappa_y}{2} \big(c_y^\top G_{yy} \ts c_y - 1\big),
\end{equation*}
where $ \kappa_x $ and $ \kappa_y $ are Lagrange multipliers. The derivatives with respect to $ c_x $ and $ c_y $ are
\begin{align*}
    \nabla_{\!c_x} \mathfrak{L}(c_x, c_y, \kappa_x, \kappa_y) = G_{xy} \ts c_y - \kappa_x \ts G_{xx} \ts c_x \stackrel{!}{=} 0, \\
    \nabla_{\!c_y} \mathfrak{L}(c_x, c_y, \kappa_x, \kappa_y) = G_{yx} \ts c_x - \kappa_y \ts G_{yy} \ts c_y \stackrel{!}{=} 0.
\end{align*}
Multiplying the first equation by $ c_x^\top $ and the second by $ c_y^\top $, it follows that $ \kappa_x = \kappa_y =: \kappa_\ell $. We thus obtain the (generalized) eigenvalue problem
\begin{equation*}
    \begin{bmatrix}
        0 & G_{xy} \\
        G_{yx} & 0
    \end{bmatrix}
    \begin{bmatrix}
        c_x \\
        c_y
    \end{bmatrix}
    =
    \kappa_\ell
    \begin{bmatrix}
        G_{xx} & 0 \\
        0 & G_{yy}
    \end{bmatrix}
    \begin{bmatrix}
        c_x \\
        c_y
    \end{bmatrix}
    \quad \text{or} \quad
    \begin{bmatrix}
        0 & G_{xx}^{-1} G_{xy} \\
        G_{yy}^{-1} G_{yx} & 0
    \end{bmatrix}
    \begin{bmatrix}
        c_x \\
        c_y
    \end{bmatrix}
    =
    \kappa_\ell
    \begin{bmatrix}
        c_x \\
        c_y
    \end{bmatrix}.
\end{equation*}
If we now solve one of the equations for $ c_x $ or $ c_y $ and insert the resulting expression into the other, we have
\begin{equation*}
    G_{xx}^{-1} G_{xy} \ts G_{yy}^{-1} G_{yx} \ts c_x = \kappa_\ell^2 \ts c_x
    \quad \text{and} \quad
    G_{yy}^{-1} G_{yx} \ts G_{xx}^{-1} G_{xy} \ts c_y = \kappa_\ell^2 \ts c_y.
\end{equation*}
Since $ G_{xx}^{-1} G_{xy} $ is a Galerkin approximation of $ \mathcal{K} $ and $ G_{yy}^{-1} G_{yx} $ a Galerkin approximation of $ \mathcal{T} $, the first expression can be used to compute eigenfunctions of $ \mathcal{F} $ and the second to compute eigenfunctions of $ \mathcal{B} $, where $ \lambda_\ell = \kappa_\ell^2 $ is the corresponding eigenvalue. This shows that the functions maximizing the correlation are the eigenfunctions of the forward--backward and backward--forward operators. We can again use the dominant eigenfunctions to compute coherent sets.

\begin{example}

\begin{figure}
    \definecolor{matlab1}{RGB}{0, 114, 189}
    \definecolor{matlab2}{RGB}{217, 83, 25}
    \definecolor{matlab3}{RGB}{237, 177, 32}
    \definecolor{matlab4}{RGB}{126, 47, 142}
    \newcommand{\cdash}[1]{\textcolor{#1}{\rule[0.5ex]{1em}{0.3ex}}}
    \centering
    \begin{minipage}[t]{0.345\textwidth}
        \centering
        \subfiguretitle{(a)}
        \includegraphics[width=\textwidth]{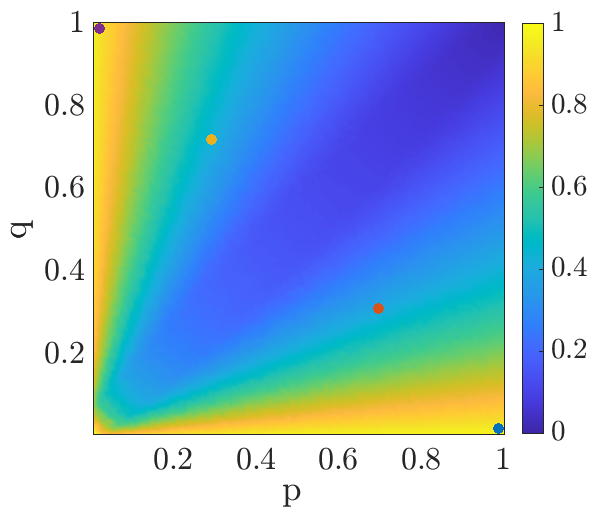}
    \end{minipage}
    \hfill
    \begin{minipage}[t]{0.345\textwidth}
        \centering
        \subfiguretitle{(b)}
        \includegraphics[width=\textwidth]{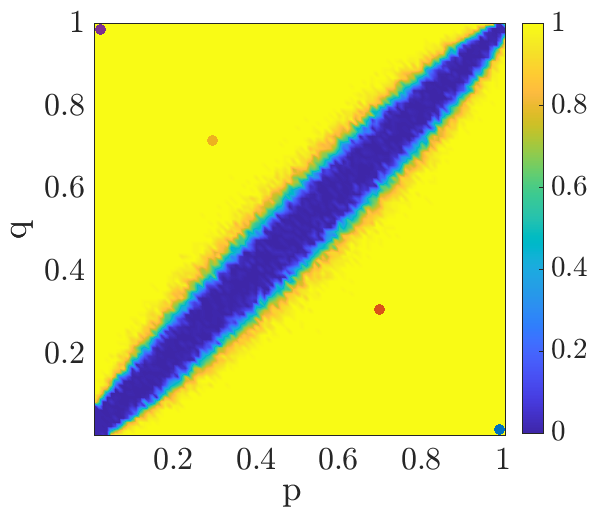}
    \end{minipage}
    \hfill
    \begin{minipage}[t]{0.263\textwidth}
        \centering
        \subfiguretitle{(c)}
        \vspace*{0.6ex}
        \includegraphics[width=\textwidth]{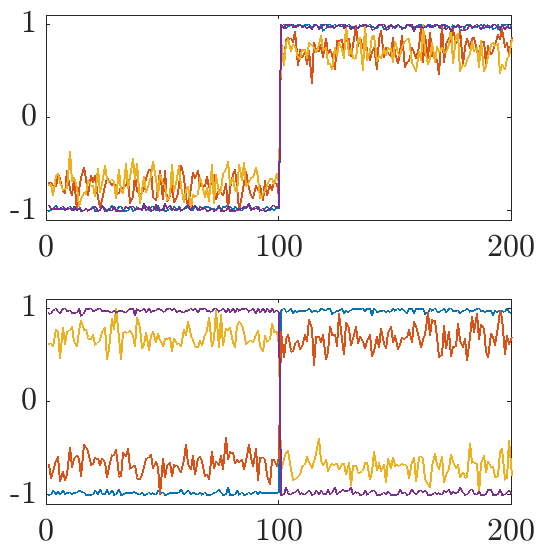}
    \end{minipage}
    \caption{(a) Correlation between $ \varphi_2 $ and $ \psi_2 $. (b) Adjusted Rand index of the obtained clustering. A value of 1 implies that the algorithm identified the correct clusters. (c)~Eigenfunctions $ \varphi_2 $ (top) and $ \psi_2 $ (bottom) for different values of $ p $ and $ q $, where \cdash{matlab1} is $p=0.99$ and $q=0.01$, \cdash{matlab2} $p=0.7$ and $q=0.3$, \cdash{matlab3} $p=0.3$ and $q=0.7$, and \cdash{matlab4} $p=0.01$ and $q=0.99$, corresponding to the dots of the same color in (a) and (b). Although the correlation decreases due to the ``noisier'' eigenvectors if $ p $ approaches $ q $, the spectral clustering algorithm still works for most cases unless $ p \approx q $. Note that the sign of $ \psi_2 $ flips once $ q > p $ (i.e., the off-diagonal blocks are now non-sparse), but the correlation increases again.}
    \label{fig:DSBM_analysis}
\end{figure}

We now generate graphs $ \mc{G} \in \mathbf{G}\big(2, 100, \big[\begin{smallmatrix} p & q \\ q & p \end{smallmatrix} \big]\big) $ with varying $ p $ and $ q $ using the DSBM. There are two dense diagonal blocks if $ p $ is large and $ q $ is small and two dense off-diagonal blocks if $ p $ is small and $ q $ large. In order to analyze the quality of the clustering for different values of $ p $ and $ q $, we plot the second-largest eigenvalue $ \kappa_2 $ (i.e., the  correlation between $ \varphi_2 $ and $ \psi_2 $) and the adjusted Rand index \cite{HA85}. The results are shown in Figure~\ref{fig:DSBM_analysis}. The clustering works for both scenarios, but, as expected, fails when the values of $ p $ and $ q $ are too similar since there are in that case no clearly defined clusters. \exampleSymbol
\end{example}

This demonstrates that the spectral clustering approach for directed graphs works for a wide range of parameter settings. Alternatively, we can rewrite \eqref{eq:CCA} as
\begin{equation} \label{eq:CCA SVD}
    \max_{\substack{\norm{\widetilde{c}_x} = 1 \\ \norm{\widetilde{c}_y} = 1}} \widetilde{c}_x^\top G_{xx}^{-\nicefrac{1}{2}} G_{xy} \ts G_{yy}^{-\nicefrac{1}{2}} \ts \widetilde{c}_y,
\end{equation}
where $ \widetilde{c}_x = G_{xx}^{\nicefrac{1}{2}} \ts c_x $ and $ \widetilde{c}_y = G_{yy}^{\nicefrac{1}{2}} \ts c_y $. The maximum of the cost function \eqref{eq:CCA SVD} is the largest singular value $ \sigma_1 $ of the matrix $ G_{xx}^{-\nicefrac{1}{2}} G_{xy} \ts G_{yy}^{-\nicefrac{1}{2}} $, attained at $ \widetilde{c}_x = u_1 $ and $ \widetilde{c}_y = v_1 $, where $ u_1 $ and $ v_1 $ are the corresponding left and right singular vectors \cite{GVL13}. Thus, setting $ c_x = G_{xx}^{-\nicefrac{1}{2}} u_1 $ and $ c_y = G_{yy}^{-\nicefrac{1}{2}} v_1 $ solves the CCA problem. This can be extended to the subsequent singular values and vectors.

\subsection{Optimization problem formulation for undirected graphs}

In order to write spectral clustering algorithms for undirected graphs as a solution of a constrained optimization problem, set $ \mu = \pi $. We then solve
\begin{equation} \label{eq:CCA symmetric}
    \max_{c \in \R^r} c^\top G_{xy} \ts c \quad \text{s.t.} \quad c^\top G_{xx} \ts c = 1,
\end{equation}
resulting in the Lagrangian function
\begin{equation*}
    \mathfrak{L}(c, \kappa) = c^\top G_{xy} \ts c - \kappa \ts \big(c^\top G_{xx} \ts c - 1\big)
\end{equation*}
and hence
\begin{equation*}
    \nabla_{\!c} \ts \mathfrak{L}(c, \kappa) = 2 \ts G_{xy} \ts c - 2 \ts \kappa \ts G_{xx} \ts c \stackrel{!}{=} 0,
\end{equation*}
where we used that $ G_{xy} $ is symmetric for $ \mu = \pi $ since due to the detailed balance condition $ C_{xy} = D_\pi \ts S = S^\top D_\pi = C_{xy}^\top $ and hence also $ G_{xy} = G_{xy}^\top $. Thus, we obtain the eigenvalue problem $ G_{xx}^{-1} G_{xy} \ts c = \kappa \ts c $. This is a Galerkin approximation of the Koopman operator and the cost function is maximized by the eigenfunction associated with the largest eigenvalue. Subdominant eigenfunctions can again be used for detecting metastable sets.

\begin{remark}
If the matrix $ G_{xy} $ is not symmetric, we can still solve the optimization problem, but would then obtain the eigenvalue problem $ \frac{1}{2} G_{xx}^{-1} (G_{xy} + G_{yx}) \ts c = \kappa \ts c $, which can be regarded as a simple form of symmetrization.
\end{remark}

Defining $ \widetilde{c} = G_{xx}^{\nicefrac{1}{2}} \ts c $, we can also write \eqref{eq:CCA symmetric} as
\begin{equation} \label{eq:CCA eig}
    \max_{\norm{\widetilde{c}} = 1} \widetilde{c}^\top G_{xx}^{-\nicefrac{1}{2}} G_{xy} \ts G_{xx}^{-\nicefrac{1}{2}} \ts \widetilde{c},
\end{equation}
which is maximized by the dominant eigenvector of the matrix $ G_{xx}^{-\nicefrac{1}{2}} G_{xy} \ts G_{xx}^{-\nicefrac{1}{2}} $. We then again obtain the eigenvalue problem $ G_{xx}^{-1} G_{xy} \ts c = \kappa \ts c $.

\subsection{Comparison}

We have shown in Section~\ref{sec:Transer operators} that our transfer operator-based spectral clustering algorithm---obtained by replacing the Koopman operator by the forward--backward operator---can be interpreted as a straightforward generalization of conventional spectral clustering algorithms for undirected graphs. This is also corroborated by the optimization problem formulation. Additionally, comparing the constrained optimization problems \eqref{eq:CCA} and \eqref{eq:CCA symmetric} for directed and undirected graphs, we see that the former is clearly more flexible and powerful since we are allowed to choose two functions, whereas the latter is restricted to one, which makes it impossible to detect dense off-diagonal blocks.

\section{Numerical results}
\label{sec:Numerical results}

We will present numerical results for a set of benchmark problems. Additional examples, including extensions to time-evolving networks, can be found in~\cite{KD22}.

\subsection{Weighted graphs with inhomogeneous degree distributions}

\begin{figure}
    \centering
    \begin{minipage}[t]{0.45\textwidth}
        \centering
        \subfiguretitle{(a)}
        \includegraphics[height=0.25\textheight]{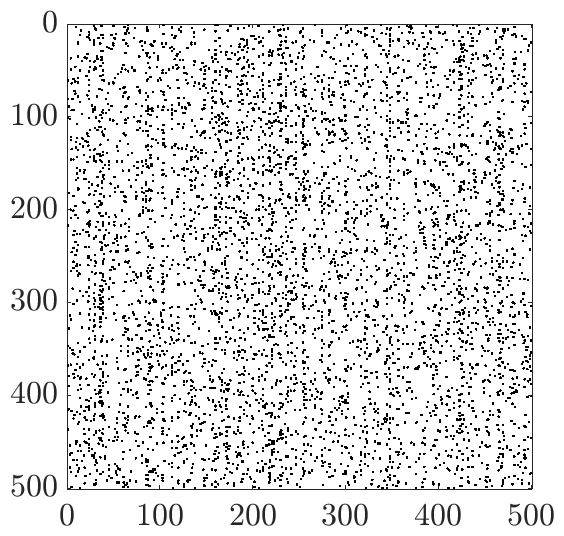}
    \end{minipage}
    \begin{minipage}[t]{0.45\textwidth}
        \centering
        \subfiguretitle{(b)}
        \includegraphics[height=0.25\textheight]{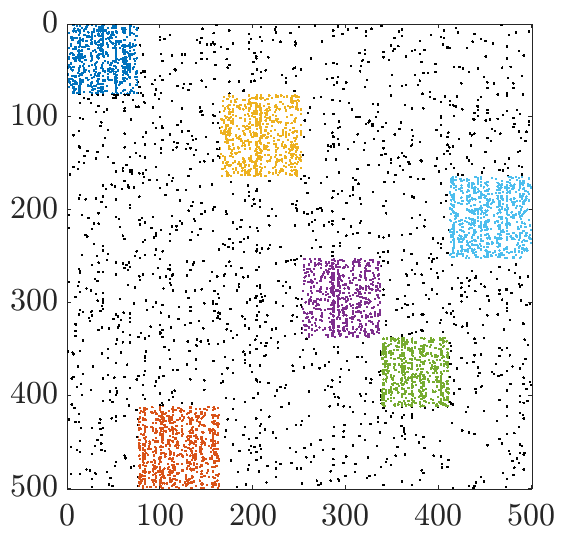}
    \end{minipage}
    \caption{(a) Adjacency matrix of a graph with six clusters. (b) Permuted adjacency matrix, where the vertices are reordered according to the identified clusters.}
    \label{fig:Benchmark graph}
\end{figure}

We generate 300 weighted graphs with heterogeneous node-degree distributions and community sizes as described in \cite{LF09} using their publicly available \href{https://www.santofortunato.net/resources}{C\texttt{++} code} (Package 4). All the graphs comprise 500 vertices and have between four and six clusters. We consider three different test cases: (1) only dense blocks on the diagonal, (2) only off-diagonal dense blocks, and (3)~a mix of both.\!\footnote{In order to generate adjacency matrices with off-diagonal blocks, we permute the block structure manually since the C\texttt{++} code does not support this feature.} A typical adjacency matrix (case 3) is shown in Figure~\ref{fig:Benchmark graph}.

\begin{table}
\caption{Comparison of the degree-discounted bibliometric symmetrization (DDBS), the Hermitian clustering approach (\textsc{Herm-RW}), and Algorithm~\ref{alg:directed} for graphs generated using a generalization of the model proposed in \cite{LF09}. For each method and test case, we compute the average adjusted Rand index (ARI) and the average number of misclassified vertices (NMV) in per cent.}
\label{tab:Benchmark graphs}
\newcolumntype{B}{X}
\scalebox{0.87}{
\renewcommand{\arraystretch}{1.2}
\begin{tabularx}{1.1\textwidth}{|B|C{1.6cm}|C{1.6cm}|C{1.6cm}|C{1.6cm}|C{1.6cm}|C{1.6cm}|C{1.6cm}|C{1.6cm}|C{1.6cm}|}
    \hline
      & \multicolumn{2}{c|}{DDBS \cite{SaPa11}}
      & \multicolumn{2}{c|}{\textsc{Herm-RW} \cite{CLSZ20}}
      & \multicolumn{2}{c|}{Algorithm~\ref{alg:directed}} \\ \hline
     & ARI & NMV & ARI & NMV & ARI & NMV \\ \hline
    diagonal blocks     & 0.954 & 3.938\,\%  & 0.059 & 68.262\,\% & 0.993 & 0.558\,\% \\
    off-diagonal blocks & 0.953 & 4.341\,\% & 0.927 & \phantom{0}6.681\,\% & 0.993 & 0.629\,\% \\
    mixed               & 0.958 & 3.742\,\%  & 0.325 & 49.933\,\% & 0.994 & 0.531\,\% \\ \hline
\end{tabularx}}
\end{table}

We apply the degree-discounted bibliometric symmetrization \cite{SaPa11}, the Hermitian clustering method proposed in \cite{CLSZ20} (both described in Remark~\ref{rem:DDBS and Hermitian clustering}), and Algorithm~\ref{alg:directed} to the three different test cases. We run each algorithm ten times (using the same $ k $-means implementation and settings) for each graph and compute the average adjusted Rand index and the average number of misclassified vertices. The results, listed in Table~\ref{tab:Benchmark graphs}, show that the Hermitian clustering approach works well for graphs with dense off-diagonal blocks, but fails to identify ``conventional'' clusters, i.e., groups of nodes that are strongly coupled by directed edges, since the method was designed to detect only imbalances in the orientation of the edges. Instead of the random-walk normalized matrix $ A_{\text{rw}} = D_\mathscr{o}^{-1} A $ as proposed in~\cite{CLSZ20}, we use the normalized adjacency matrix $ A_{\text{nn}} = D_\mathscr{o}^{-\nicefrac{1}{2}} A D_\mathscr{i}^{-\nicefrac{1}{2}} $, which seems to improve the results. Although the reweighting parameters for the degree-discounted bibliometric symmetrization were determined empirically, the method works for all test cases. The transfer operator-based clustering also works well for all considered graph types and leads to marginally better results. The advantage of the dynamical systems approach is that it offers a rigorous interpretation of clusters in terms of coherent sets (or canonical correlation analysis) and that it can be easily extended to dynamic graphs. For more complicated benchmark problems, it might also be beneficial to tune the probability density $ \mu $ or to use a combination of the operators $ \mathcal{F} $ and $ \mathcal{B} $.

\subsection{32-bit adder}

As a last example, we consider a 32-bit adder. The data set can be found on the  \href{https://math.nist.gov/MatrixMarket/data/misc/hamm/add32.html}{Matrix Market} website. Since the matrix contains negative values, we shift the entries so that all weights are positive. The circuit has a fairly simple structure and the spectral clustering algorithm detects 32 blocks of nearly the same size. Reordering the adjacency matrix based on the cluster numbers, we obtain a block matrix with very few nonzero entries in the off-diagonal blocks as shown in Figure~\ref{fig:32-bit adder}. The bandwidth of the matrix could be minimized by applying generalizations of the  Cuthill--McKee algorithm to the block structure \cite{ReSc06}.

\begin{figure}
    \centering
    \begin{minipage}[t]{0.326\textwidth}
        \centering
        \subfiguretitle{(a)}
        \includegraphics[height=0.19\textheight]{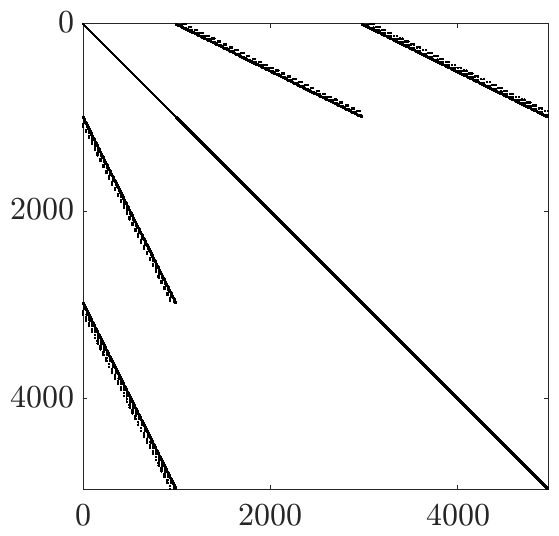}
    \end{minipage}
    \begin{minipage}[t]{0.326\textwidth}
        \centering
        \subfiguretitle{(b)}
        \vspace*{0.6ex}
        \includegraphics[height=0.204\textheight]{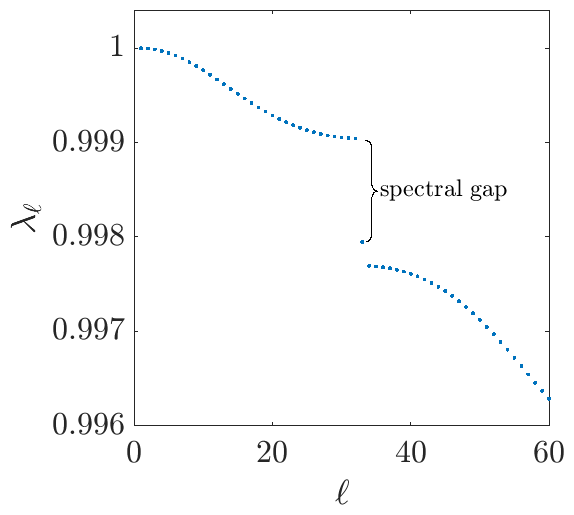}
    \end{minipage}
    \begin{minipage}[t]{0.326\textwidth}
        \centering
        \subfiguretitle{(c)}
        \includegraphics[height=0.19\textheight]{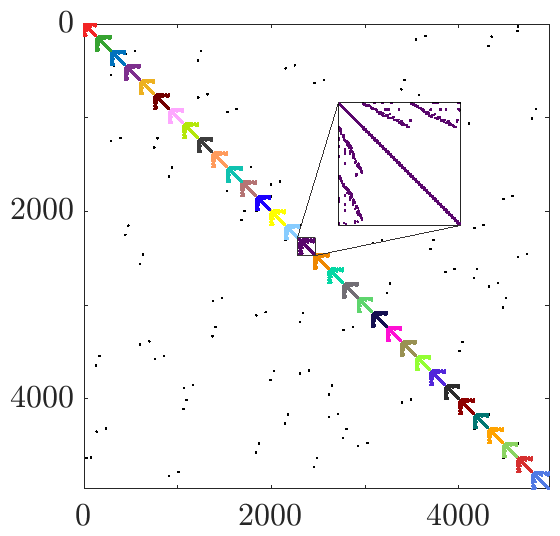}
    \end{minipage}
    \caption{(a)~Structure of the original adjacency matrix of the 32-bit adder. (b)~Dominant eigenvalues of the forward--backward operator. There is a spectral gap between the 32nd and 33rd eigenvalue. We thus set $ k = 32 $. (c). Permuted adjacency matrix, where we again reorder the vertices according to the identified clusters.}
    \label{fig:32-bit adder}
\end{figure}

\section{Conclusion}
\label{sec:Conclusion}

We have defined different types of transfer operators associated with random walks on graphs and shown that the eigenfunctions contain information about metastable or coherent sets. The main advantages of the dynamical systems perspective are that the derived algorithms can be interpreted as generalizations of conventional spectral clustering methods for undirected graphs and that they can also be easily extended to time-evolving graphs as shown in \cite{KD22}. Future work includes a thorough analysis of the proposed coherent set-based spectral clustering approach for dynamic graphs. Of particular interest are graphs comprising clusters whose structure changes in time, e.g., splitting and merging as well as shrinking and growing clusters. First results presented in \cite{KD22} illustrate that it is possible to detect time-dependent clusters in dynamic graphs that are based on discretized dynamical systems with coherent sets. However, finding such clusters in real-world graphs, which may have a more intricate and less homogeneous coupling structure, might require tuning and improving the algorithms. Also the properties of transfer operators associated with time-evolving graphs need to be studied in more detail. An open question is whether it is possible to derive related optimization problems that take each layer of the network into account. Furthermore, it would be interesting to interpret the optimization problems derived above as relaxed versions of discrete counterparts. How is this related to normalized cuts and can these problems be efficiently solved using quantum annealers?

Another research problem is the derivation of suitable basis functions for the Galerkin approximation that would allow us to coarse-grain the system without losing essential information about the metastable or coherent sets. While it is theoretically possible to reduce the size of the eigenvalue problem significantly, computing the projected operator and solving the resulting less sparse eigenvalue problem numerically might not be advantageous in practice.

For the benchmark problems in Section~\ref{sec:Numerical results}, we simply computed the forward--backward operator with respect to the uniform density. However, it might be beneficial to choose different densities $ \mu $ to counterbalance the inhomogeneous degree distributions, similar to the degree-discounted symmetrization proposed in \cite{SaPa11}. Using standard $ k $-means, we assign each vertex to a cluster. However, some vertices might not belong to any cluster or to multiple clusters at the same time. Soft clustering methods might be able to detect overlapping communities or transition regions.

\section*{Data availability}

The spectral clustering code and examples that support the findings presented in this paper are available at \url{https://github.com/sklus/d3s/}.

\section*{Acknowledgments}

M.\ T. was supported by the EPSRC Centre for Doctoral Training in Mathematical Modelling, Analysis and Computation (MAC-MIGS) funded by the UK Engineering and Physical Sciences Research Council (grant EP/S023291/1), Heriot--Watt University and the University of Edinburgh. We would also like to thank Nata\v sa Djurdjevac Conrad for many fruitful discussions and helpful suggestions.

\bibliographystyle{unsrturl}
\bibliography{TOSCA}

\appendix

\section{Properties of transfer operators}
\label{app:Properties of transfer operators}

We are mainly interested in spectral properties and how they are related to inherent time scales and metastable or coherent sets.

\begin{proposition} \label{pro:operator properties}
The operators defined above have the following properties:
\begin{enumerate}[leftmargin=4ex, itemsep=0ex, topsep=0.5ex, label=\roman*), beginpenalty=10000]
\item It holds that $ \innerprod{\mathcal{T} u}{f}_\nu = \innerprod{u}{\mathcal{K} f}_\mu $.
\item The operator $ \mathcal{F} $ is self-adjoint and positive semi-definite w.r.t.\ $ \innerprod{\cdot}{\cdot}_\mu $.
\item The operator $ \mathcal{B} $ is self-adjoint and positive semi-definite w.r.t.\ $ \innerprod{\cdot}{\cdot}_\nu $.
\item $ \mathcal{F} \mathds{1} = \mathds{1} $ and $ \mathcal{B} \mathds{1} = \mathds{1} $, i.e., $ F $ and $ B $ are row-stochastic.
\item The eigenvalues of $ \mathcal{F} $ and $ \mathcal{B} $ are real-valued and contained in the interval $ [0, 1] $.
\end{enumerate}
\end{proposition}

\begin{proof}
Note that the $ \mu $-weighted inner product can be written as $ \innerprod{f}{g}_\mu = \mathbf{f}^\top D_\mu \ts \mathbf{g} $.
\begin{enumerate}[leftmargin=3.5ex, itemsep=0ex, topsep=0.5ex, label=\roman*)]
\item Then
\begin{equation*}
    \innerprod{\mathcal{T} u}{f}_\nu
        = \big(D_\nu^{-1} \ts S^\top D_\mu \ts \boldsymbol{u}\big)^\top D_\nu \ts \boldsymbol{f}
        = \boldsymbol{u}^\top D_\mu \ts S \boldsymbol{f}
        = \innerprod{u}{\mathcal{K} f}_\mu.
\end{equation*}
\item Additionally, we have
\begin{equation*}
    \innerprod{\mathcal{F} u}{f}_\mu
        = \big(S D_\nu^{-1} \ts S^\top D_\mu \ts \boldsymbol{u}\big)^\top D_\mu \ts \boldsymbol{f}
        = \boldsymbol{u}^\top D_\mu \big(S \ts D_\nu^{-1} \ts S^\top D_\mu \ts \boldsymbol{f}\big)
        = \innerprod{u}{\mathcal{F} f}_\mu.
\end{equation*}
Furthermore, $ \innerprod{\mathcal{F} u}{u}_\mu = \big\| D_\nu^{-\nicefrac{1}{2}} S^\top D_\mu u \ts \big\|^2 \ge 0 $.
\item Analogous to ii).
\item First, $ \mathcal{T} \mathds{1} = D_\nu^{-1} \ts S^\top D_\mu \ts \mathds{1} = D_\nu^{-1} \ts S^\top \boldsymbol{\mu} = D_\nu^{-1} \ts \boldsymbol{\nu} = \mathds{1} $. Also, $ \mathcal{K} \mathds{1} = S \ts \mathds{1} = \mathds{1} $ since $ S $ is row-stochastic. Thus, this also holds for compositions of these operators.
\item Using ii) and iii), the eigenvalues are real-valued and non-negative. Furthermore, since $ F $ and $ B $ are stochastic matrices as shown in iv), the spectral radius is~$ 1 $. \qedhere
\end{enumerate}
\end{proof}

\begin{corollary} \label{cor:doubly stochastic}
If $ \mu $ is the uniform distribution, then $ F $ is doubly stochastic. Analogously, if $ \nu $ is the uniform distribution, then $ B $ is doubly stochastic. Since $ \mu = S^{\ts-\!\top} \nu $, it follows that $ B $ is doubly stochastic if $ X \sim \frac{1}{n} S^{\ts-\!\top} \mathds{1} $.
\end{corollary}

\begin{proof}
By Proposition~\ref{pro:operator properties}, $ F $ and $ B $ are row-stochastic. If $ \mu $ is the uniform distribution, then $ F $ is symmetric. Similarly, $ B $ is symmetric if $ \nu $ is the uniform distribution.
\end{proof}

All the transfer operators introduced above can also be represented as compositions of covariance and cross-covariance operators.

\begin{proposition} \label{pro:covariance operator representations}
Let $ \mathcal{C}_{yx} = \mathcal{C}_{xy}^* $. Then it holds that:
\begin{enumerate}[leftmargin=4ex, itemsep=0ex, topsep=0.5ex, label=\roman*)]
\item $ \mathcal{K} = \mathcal{C}_{xx}^{-1} \mathcal{C}_{xy} $,
\item $ \mathcal{P} = \mathcal{C}_{xx}^{-1} \mathcal{C}_{yx} $, provided that $ \mu $ is the uniform distribution,
\item $ \mathcal{T} = \mathcal{C}_{yy}^{-1} \mathcal{C}_{yx} $,
\item $ \mathcal{F} = \mathcal{C}_{xx}^{-1} \mathcal{C}_{xy} \ts \mathcal{C}_{yy}^{-1} \mathcal{C}_{yx} $,
\item $ \mathcal{B} = \mathcal{C}_{yy}^{-1} \mathcal{C}_{yx} \ts \mathcal{C}_{xx}^{-1} \mathcal{C}_{xy} $.
\end{enumerate}
\end{proposition}

\begin{proof}
We can either use properties of the operators or their matrix representations:
\begin{enumerate}[leftmargin=3.5ex, itemsep=0ex, topsep=0.5ex, label=\roman*)]
\item This follows from $ \innerprod{f}{\mathcal{C}_{xy} \ts g} = \innerprod{f}{\mathcal{K} g}_\mu = \innerprod{f}{\mathcal{C}_{xx} \ts \mathcal{K} g} $ or $ C_{xx}^{-1} C_{xy} = S = K $.
\item Using the duality between $ \mathcal{P} $ and $ \mathcal{K} $, we write
\begin{equation*}
    \innerprod{\mathcal{P} \rho}{f}
        = \innerprod{\rho}{\mathcal{C}_{xx}^{-1} \mathcal{C}_{xy} f}
        = \innerprod{\mathcal{C}_{yx} \ts \mathcal{C}_{xx}^{-1} \rho}{f}
        = \innerprod{\mathcal{C}_{xx}^{-1} \mathcal{C}_{yx} \rho}{f},
\end{equation*}
where $ \mathcal{C}_{xx}^{-1} $ and $ \mathcal{C}_{yx} $ commute if $ \mu $ is the uniform distribution. Alternatively, this can be seen as follows: $ C_{xx}^{-1} C_{yx} = D_\mu^{-1} S^\top D_\mu = n \ts I \ts S^\top \frac{1}{n} I = S^\top = P $.
\item By Proposition~\ref{pro:operator properties}, we have
\begin{equation*}
    \innerprod{\mathcal{T} u}{f}_\nu
        = \innerprod{u}{\mathcal{K} f}_\mu
        = \innerprod{u}{\mathcal{C}_{xx}^{-1} \mathcal{C}_{xy} f}_\mu
        = \innerprod{u}{\mathcal{C}_{xy} f}
        = \innerprod{\mathcal{C}_{yx} u}{f}
        = \innerprod{\mathcal{C}_{yy}^{-1} \mathcal{C}_{yx} u}{f}_\nu.
\end{equation*}
Similarly, $ C_{yy}^{-1} C_{yx} = D_\nu^{-1} S^\top D_\mu = T $.
\item and v) This follows immediately from the definitions of $ \mathcal{F} $ and $ \mathcal{B} $. \qedhere
\end{enumerate}
\end{proof}

\end{document}